\documentclass[10pt,letter]{article}
\usepackage[margin=3cm]{geometry}

\usepackage{graphicx}
\usepackage{graphics}
\usepackage{amsthm}
\usepackage{amsmath}
\usepackage{url}
\usepackage{bm}
\usepackage{subfigure}
\usepackage[ruled]{algorithm2e}
\usepackage{color}
\usepackage{multirow}

\newtheorem{theorem}{Theorem}[section]
\newtheorem{lemma}[theorem]{Lemma}

\newcommand{\ba}{\begin{array}}
\newcommand{\ea}{\end{array}}
\newcommand{\beq}{\begin{equation}}
\newcommand{\eeq}{\end{equation}}
\newcommand{\beqa}{\begin{eqnarray}}
\newcommand{\eeqa}{\end{eqnarray}}
\newcommand{\beqas}{\begin{eqnarray*}}
\newcommand{\eeqas}{\end{eqnarray*}}
\newcommand{\bi}{\begin{itemize}}
\newcommand{\ei}{\end{itemize}}
\newcommand{\gap}{\hspace*{2em}}

\newcommand{\nn}{\nonumber}

\def\eqref#1{(\ref{#1})}

\def\bfLambda{{\mathbf \Lambda}}

\def\bfS{{\bf S}}
\def\bfU{{\bf U}}
\def\bfX{{\bf X}}
\def\bfY{{\bf Y}}

\def\cL{{\cal L}}

\def\cS{{\cal S}}

\def\cU{{\cal U}}

\def\diag{{\rm diag}}
\def\Diag{{\rm Diag}}

\def\htheta{{\widehat {\mathbf \Theta}}}
\def\bfTheta{{\mathbf \Theta}}
\def\hw{{\widehat\mathbf{W}}}

\def\tr{{\rm tr}}

\def\vec{{\rm vec}}

\def\bI{{\bar I}}
\def\bJ{{\bar J}}
\def\ext{{\rm ext}}

\SetKwRepeat{doWhile}{for $j=1,2,\dots,p,1,2,\dots,p,\dots$}{Until}

\title{ Fused Multiple Graphical Lasso}

\author{
Sen Yang\thanks{School of Computing, Informatics, and Decision Systems Engineering, Arizona State University, Tempe, Arizona 85287, USA ({senyang@asu.edu, peter.wonka@asu.edu, jieping.ye@asu.edu}).},~  
Zhaosong Lu\thanks{Department of Mathematics, Simon Frasor University, Burnaby, BC, V5A 156, Canada (zhaosong@sfu.ca).},~   
Xiaotong Shen\thanks{School of Statistics, University of Minnesota, Minneapolis, Minnesota 55455, USA (xshen@umn.edu).},~ 
Peter Wonka\footnotemark[1],~ 
and Jieping Ye\footnotemark[1]
}
\begin{document}

\maketitle

\begin{abstract}
In this paper, we consider the problem of estimating multiple graphical models simultaneously using the fused lasso penalty, which  encourages adjacent graphs to share similar structures. A motivating example is the analysis of brain networks of Alzheimer's disease using neuroimaging data. Specifically, we may wish to estimate a brain network for the normal controls (NC), a brain network for the patients with mild cognitive impairment (MCI), and a brain network for Alzheimer's patients (AD). We expect the two brain networks for NC and MCI to share common structures but not to be identical to each other; similarly for the two brain networks for MCI and AD. The proposed formulation can be solved using a second-order method. Our key technical contribution is to establish the necessary and sufficient condition for the graphs to be decomposable. Based on this key property, a simple screening rule is presented, which decomposes the large graphs into small subgraphs and allows an efficient estimation of multiple independent (small) subgraphs, dramatically reducing the computational cost. We perform experiments on both synthetic and real data; our results demonstrate the effectiveness and efficiency of the proposed approach.
\end{abstract}
\section{Introduction}
Undirected graphical models explore the relationships among a set of random variables through their joint distribution. The estimation of undirected graphical models has applications in many domains, such as computer vision, biology, and medicine \cite{guo2011joint,huang2009learning,yang2012feature}. One instance is the analysis of gene expression data. {As shown in many biological studies, genes tend to work in groups based on their biological functions, and there exist some regulatory relationships between genes~\cite{chuang2007network}}. Such biological knowledge can be represented as a graph, where nodes are the genes, and edges describe the regulatory relationships. Graphical models provide a useful tool for modeling these relationships, and can be used to explore gene activities. One of the most widely used graphical models is the Gaussian graphical model (GGM), which assumes the variables to be Gaussian distributed~\cite{banerjee2008model,yuan2007model}. In the framework of GGM, the problem of learning a graph is equivalent to estimating the inverse of the covariance matrix (precision matrix), since the nonzero off-diagonal elements of the precision matrix represent edges in the graph~\cite{banerjee2008model,yuan2007model}.

In recent years many research efforts have focused on estimating the precision matrix and the corresponding graphical model (see, for example \cite{banerjee2008model,GLasso,hsieh2011sparse,huang2009learning,li2010inexact,
liu2010stability,lu2009smooth,lu2010adaptive,mazumder2012graphical,meinshausen2006high,olsen2012newton,yuan2007model}. Meinshausen and B\"{u}hlmann \cite{meinshausen2006high} estimated edges for each node in the graph by fitting a lasso problem \cite{lasso} using the remaining variables as predictors. Yuan and Lin \cite {yuan2007model} and Banerjee et al. \cite{banerjee2008model} proposed a penalized maximum likelihood model using $\ell_1$ regularization to estimate the sparse precision matrix. Numerous methods have been developed for solving this model.  For example,  d'Aspremont et al.~\cite{d2008first} and Lu \cite{lu2009smooth,lu2010adaptive} studied Nesterov's smooth gradient
methods \cite{nesterov2005smooth} for solving this problem or its dual.  Banerjee et al.~\cite{banerjee2008model} and Friedman et al.~\cite{GLasso}  proposed  block coordinate ascent methods for solving the dual problem. The latter method \cite{GLasso} is widely referred to as Graphical lasso (GLasso). Mazumder and Hastie \cite{mazumder2012graphical} proposed a new algorithm called DP-GLasso, each step of which is a box-constrained QP problem. Scheinberg and Rish \cite{scheinberg2009sinco} proposed a coordinate descent method for solving this model in a greedy approach. Yuan \cite{yuan2012alternating} and Scheinberg et al.~\cite{scheinberg2010sparse} applied alternating direction method of multipliers (ADMM) \cite{boyd2011distributed} to this problem. {Li and Toh \cite{li2010inexact} and Yuan and Lin \cite{yuan2007model} proposed to solve this problem using interior point methods.} Wang et al. \cite{wang2010solving}, Hsieh et al.~\cite{hsieh2011sparse}, Olsen et al.~\cite{olsen2012newton}, and Dinh et al.~\cite{dinh2013proximal} studied Newton method for solving
this model. 
The main challenge of estimating a sparse precision matrix for the problems with a large number of nodes (variables) is its intensive computation. Witten et al.~\cite{witten2011new} and Mazumder and Hastie~\cite{mazumder2012exact} independently derived {a necessary and sufficient condition for the solution of a single graphical lasso to be block diagonal {(subject to some rearrangement of variables)}. This can be used as a simple screening test to identify the associated blocks, and the original problem can thus be decomposed into a group of smaller sized but independent problems corresponding to these blocks. When the number of blocks is large, it can achieve massive computational gain.} However, these formulations assume that observations are independently drawn from a single Gaussian distribution. In many applications the observations may be drawn from multiple Gaussian distributions; in this case, multiple graphical models need to be estimated.

There are some recent works on the estimation of multiple precision matrices~\cite{Danaher2011joint,guo2011joint,hara2011common,honorio2010multi,kolar2010estimating,kolar2011time,mohan2012structured,zhou2008time}. Guo et al.~\cite{guo2011joint} proposed a method to jointly estimate multiple graphical models using a hierarchical penalty. However, their model is not convex. Honorio and Samaras \cite{honorio2010multi} proposed a convex formulation to estimate multiple graphical models using the $\ell_{1,\infty}$ regularizer. Hara and Washio \cite{hara2011common} introduced a method to learn common substructures among multiple graphical models. Danaher et al.~\cite{Danaher2011joint} estimated multiple precision matrices simultaneously using a pairwise fused penalty and grouping penalty. ADMM was used to solve the problem, but it requires computing multiple eigen decompositions at each iteration. {Mohan et al. \cite{mohan2012structured} proposed to estimate multiple precision matrices based on the assumption that the network differences are generated from node perturbations. Compared with single graphical model learning, learning multiple precision matrices jointly is even more challenging to solve. Recently, a necessary and sufficient condition for multiple graphs to be decomposable was proposed in~\cite{Danaher2011joint}. However, such necessary and sufficient condition was restricted to two graphs only when the fused penalty is used. It is not clear whether this screening rule can be extended to the more general case with more than two graphs, which is the case in brain network modeling.}

{There are two types of fused penalties that can be used for estimating multiple (more than two)
graphs: (a) pairwise fused or (b) sequential fused~\cite{tibshirani2005sparsity}. In this paper we set out to address the sequential fused case first, because we work on practical applications that can be more appropriately
formulated using the sequential formulation. Specifically, }
we consider the problem of estimating multiple graphical models by maximizing a penalized log likelihood with $\ell_1$ and sequential fused regularization. The $\ell_1$ regularization yields a sparse solution, and the fused regularization encourages adjacent graphs to be similar. The graphs considered in this paper have a natural order, which
is common in many applications. A motivating example is the modeling of brain networks for Alzheimer's disease using neuroimaging data such as Positron emission tomography (PET). In this case, we want to estimate graphical models for three groups: normal controls (NC), patients of mild cognitive impairment (MCI), and Alzheimer's patients (AD). These networks are expected to share some common connections, but they are not identical. Furthermore, the networks are expected to evolve over time, in the order of disease progression from NC to MCI to AD. Estimating the graphical models separately fails to exploit the common structures among them. It is thus desirable to jointly estimate the three networks (graphs).
Our key technical contribution is to establish the necessary and sufficient condition for the solution of the fused multiple graphical lasso (FMGL) to be block diagonal. The duality theory and several other tools in linear programming are used to drive
the necessary and sufficient condition. Based on this crucial property of FMGL, we develop a screening rule
which enables the efficient estimation of large multiple precision matrices for FMGL. The proposed screening rule can be combined with any algorithms to reduce computational cost. We employ {a second-order method \cite{hsieh2011sparse,lee2012proximal,tseng2009coordinate}} to solve the fused multiple graphical lasso, where each step is solved by the spectral projected gradient method~\cite{lu2011augmented,wright2009sparse}. In addition, we propose a shrinking scheme to identify the variables to be updated in each step of the second-order method, which reduces the computation cost of each step. We conduct experiments on both synthetic and real data; our results demonstrate the effectiveness and efficiency of the proposed approach.

\subsection{Notation} \label{notation}

In this paper, $\Re$ stands for the set of all real numbers, $\Re^n$ denotes the $n$-dimensional Euclidean space, and the set of all $m \times n$ matrices with real entries is denoted by
$\Re^{m \times n}$.  All matrices are presented in bold format. The space of  symmetric matrices is denoted by $\cS^n$. If $\bfX \in \cS^n$ is positive semidefinite (resp.\ definite), we write $\bfX \succeq 0$ (resp.\ $\bfX \succ 0$).  Also, we write $\bfX \succeq \bfY$ to mean $\bfX-\bfY \succeq 0$. The cone of positive semidefinite matrices in $\cS^n$ is denoted by $\cS^n_+$.
Given matrices $\bfX$ and $\bfY$ in $\Re^{m \times n}$, the standard inner product is defined by $\langle \bfX, \bfY \rangle := \tr (\bfX \bfY^T)$, where $\tr(\cdot)$ denotes the trace of a matrix. $\bfX \circ \bfY$ and $\bfX\otimes\bfY$ means the Hadamard and Kronecker product of $\bfX$ and $\bfY$, respectively. We denote the identity
matrix by $\bf I$, whose dimension should be clear from the context. The determinant and the minimal eigenvalue
of a real symmetric matrix $\bfX$ are denoted by $\det(\bfX)$ and $\lambda_{\min}(\bfX)$,
respectively. Given a matrix $\bfX \in \Re^{n \times n}$, $\diag(\bfX)$ denotes the vector formed by
the diagonal of $\bfX$, that is, $\diag(\bfX)_{i}=\bfX_{ii}$ for $i=1,\ldots,n$. $\Diag(\bfX)$ is the
diagonal matrix which shares the same diagonal as $\bfX$. $\vec(\bfX)$ is the vectorization of $\bfX$. In addition, $\bfX >0$ means that all entries
of $\bfX$ are positive.

The rest of the paper is organized as follows. We introduce the fused multiple graphical lasso formulation in Section~\ref{sec:fusedMultiGLasso}. The screening rule is presented in Section~\ref{sec:screening}. The proposed second-order method is presented in Section~\ref{sec:Newtonmethod}. The experimental results are shown in Section~\ref{sec:expR}. We conclude the paper in Section~\ref{sec:discussion}.
\section{Fused multiple graphical lasso}\label{sec:fusedMultiGLasso}
Assume we are given $K$ data sets, $x^{(k)}\in \mathcal{R}^{n_k\times p},~k=1,\dots,K$ with $K\geq 2$, where $n_k$ is the number of samples, and $p$ is the number of features. The $p$ features are common for all $K$ data sets, and all $\sum_{k=1}^Kn_k$ samples are independent. Furthermore, the samples within each data set $x^{(k)}$ are identically distributed with a $p$-variate Gaussian distribution with zero mean and positive definite covariance matrix $\mathbf{\Sigma}^{(k)}$, and there are many conditionally independent pairs of features, i.e., the precision matrix $\mathbf{\Theta}^{(k)} = (\mathbf{\Sigma}^{(k)})^{-1}$ should be sparse. For notational simplicity, we assume that $n_{1}=\dots = n_K = n$. Denote the sample covariance matrix for each data set $x^{(k)}$ as $\mathbf{S}^{(k)}$ with $\mathbf{S}^{(k)}=\frac{1}{n}(x^{(k)})^Tx^{(k)}$, and $\mathbf{\Theta}=(\mathbf{\Theta}^{(1)},\dots,\mathbf{\Theta}^{(K)})$.
Then the negative log likelihood for the data takes the form of
\begin{equation}\label{eq:loglikelihood}
  \sum_{k=1}^K\left( - \log \det({\mathbf{\Theta}^{(k)}})+\tr(\mathbf{S}^{(k)}{\mathbf{\Theta}^{(k)}})\right).
\end{equation}
Clearly, minimizing \eqref{eq:loglikelihood} leads to the maximum likelihood estimate (MLE) $\mathbf{\hat \Theta}^{(k)} = (\mathbf{S}^{(k)})^{-1}$.
{However, the MLE fails when $\mathbf{S}^{(k)}$ is singular. Furthermore, the MLE is usually dense. }The $\ell_1$ regularization has been employed to induce sparsity, resulting in the sparse inverse covariance estimation~\cite{banerjee2008model,GLasso,yuan2006model}. In this paper, we employ both the $\ell_1$ regularization and the fused regularization for simultaneously estimating multiple graphs. The $\ell_1$ regularization leads to a sparse solution, and the fused penalty encourages $\mathbf{\Theta}^{(k)}$ to be similar to its neighbors. Mathematically, we solve the following formulation:
\begin{equation}\label{eq:fusedloglikelihood}
\min_{\mathbf{\Theta}^{(k)}\succ 0, k=1\dots K} \sum_{k=1}^K\left( - \log \det({\mathbf{\Theta}^{(k)}})+\tr(\mathbf{S}^{(k)}{\mathbf{\Theta}^{(k)}})\right) + P({\mathbf{\Theta}}),
\end{equation}
where
\[
P({\mathbf \Theta}) = \lambda_1 \sum_{k=1}^K\sum_{i\neq j}|{\mathbf \Theta}_{ij}^{(k)}| + \lambda_2 \sum_{k=1}^{K-1}\sum_{i\neq j}|{\mathbf \Theta}_{ij}^{(k)} - {\mathbf \Theta}_{ij}^{(k+1)}|,
\]
$\lambda_1>0$ and $\lambda_2>0$ are positive regularization parameters. This model is referred to as  the fused multiple graphical lasso (FMGL).

To ensure the existence of a solution for problem \eqref{eq:fusedloglikelihood}, we assume throughout this paper that $\diag(\mathbf{S}^{(k)})>0, k=1,\dots, K$. Recall that $\mathbf{S}^{(k)}$ is a sample covariance matrix, and hence $\diag(\mathbf{S}^{(k)})\geq 0$. The diagonal entries may be not, however, strictly positive. But we can always add a small perturbation (say $10^{-8}$) to ensure the above assumption holds.

The following theorem shows that under this assumption the FMGL \eqref{eq:fusedloglikelihood} has a unique solution.

\begin{theorem}\label{thm:coord_desc_thm}
{Under the assumption that $\diag(\mathbf{S}^{(k)})>0, k=1,\dots, K$, problem \eqref{eq:fusedloglikelihood} has a unique optimal solution.}
\end{theorem}

To prove Theorem \ref{thm:coord_desc_thm}, we first establish a technical lemma which regards the existence of a solution for a standard graphical lasso problem.

\begin{lemma} \label{single-graph}
Let $\mathbf{S} \in \cS^p_+$ and $\bfLambda \in \cS^p$ be such that $\Diag(\bfS)+\bfLambda>0$ and $\diag(\bfLambda)\geq 0$.
Consider the problem
\beq \label{logdet}
\min\limits_{\mathbf{X} \succ 0} \underbrace{-\log\det(\mathbf{X}) + \tr(\mathbf{S}\mathbf{X}) + \sum_{ij}\bfLambda_{ij}  |\bfX_{ij}|}_{f(\bfX)}.
\eeq
Then the following statements hold:
\bi
\item[(a)] Problem~\eqref{logdet}  has a unique optimal solution;
\item[(b)] The sub-level set $\cL = \{\bfX \succ 0: f(\bfX) \le \alpha\}$ is compact for any $\alpha \ge f^*$, where $f^*$ is the optimal value of ~\eqref{logdet}.
\ei
\end{lemma}

\begin{proof}
(a) Let $\cU = \{\bfU \in \cS^p: \bfU_{ij} \in [-1,1], \  \forall i, j\}$. Consider the problem
\beq \label{dual-prob}
 \max\limits_{\bfU \in \cU} \left\{\log\det(\mathbf{S}+\bfLambda \circ \bfU):
\mathbf{S}+\bfLambda \circ \bfU \succ 0\right\}.
\eeq

We first claim that the feasible region of problem \eqref{dual-prob} is  nonempty, or equivalently, there
exists $\bar\bfU \in \cU$ such that $\lambda_{\min}(\bfS+\bfLambda \circ\bar\bfU) >0$. Indeed, one
can observe that
\beqa
\max\limits_{\bfU \in \cU} \lambda_{\min}(\bfS+\bfLambda \circ\bfU) &=& \max\limits_{t, \bfU \in \cU}
\{t: \bfLambda \circ \bfU +\bfS -t \mathbf{I} \succeq 0 \},  \nn\\
&=& \min\limits_{\bfX \succeq 0} \max\limits_{t,\bfU\in \cU} \left\{t + \tr(\bfX(\bfLambda \circ \bfU +\bfS -t \mathbf{I})) \right\}, \nn \\
&=& \min\limits_{\bfX \succeq 0} \left\{\tr(\bfS\bfX)+\sum\limits_{ij}\bfLambda_{ij}|\bfX_{ij}|: \ \tr(\bfX)=1\right\}, \label{lambdamin}
\eeqa
where the second equality follows from the Lagrangian duality since its associated Slater condition is
satisfied. Let $\Omega: = \{\bfX \in \cS^p: \tr(\bfX)=1, \ \bfX \succeq 0\}$.  By the assumption
$\Diag(\bfS)+\bfLambda>0$, we see that $\bfLambda_{ij}  > 0$ for all $i\neq j$ and $\bfS_{ii}+\bfLambda_{ii}>0$ for every $i$.  Since $\Omega \subset \cS^p_+$, we have
$\tr(\bfS\bfX) \ge 0$ for all $\bfX \in \Omega$. If there exists some $k\neq l$ such that
$\bfX_{kl}>0$, then $\sum\limits_{i\neq j} \bfLambda_{ij}|\bfX_{ij}| >0$ and hence,
\beq \label{sx-ineq}
\tr(\bfS\bfX)+\sum_{ij}\bfLambda_{ij}  |\bfX_{ij}| ~>~ 0, ~ \forall \bfX \in \Omega.
\eeq
 Otherwise, one has $\bfX_{ij}=0$ for all $i\neq j$, which, together with the facts that $\bfS_{ii}+\bfLambda_{ii}>0$ for all $i$  and $\tr(\bfX)=1$, implies that for all $\bfX \in \Omega$,
\[
\tr(\bfS\bfX) +\sum\limits_{ij}\bfLambda_{ij}|\bfX_{ij}| ~=~ \sum\limits_{i} (\bfS_{ii}+\bfLambda_{ii})\bfX_{ii} ~\ge~ \tr(\bfX) \min\limits_i (\bfS_{ii}+\bfLambda_{ii}) ~>~ 0.
\]
Hence,  \eqref{sx-ineq} again holds. Combining \eqref{lambdamin} with \eqref{sx-ineq}, one can see that $\max\limits_{\bfU \in \cU} \lambda_{\min}(\bfS+\bfLambda \circ \bfU)~>~ 0$. Therefore, problem \eqref{dual-prob} has at least a feasible solution.

We next show that problem \eqref{dual-prob} has an optimal solution.
Let $\bar\bfU$ be a feasible point of  \eqref{dual-prob}, and
\[
\bar\Omega := \{\bfU \in \cU:
\log\det(\mathbf{S}+\bfLambda \circ \bfU) ~\ge~ \log\det(\mathbf{S}+\bfLambda \circ \bar \bfU), \ \mathbf{S}+\bfLambda \circ \bfU \succ 0\}.
\]
One can observe that $\{\mathbf{S}+\bfLambda \circ \bfU: \bfU \in \cU\}$ is compact. Using this fact, it is not hard to see that $\log\det(\mathbf{S}+\bfLambda \circ \bfU) \to -\infty$ as $\bfU\in\cU$ and $\lambda_{\min}(\mathbf{S}+\bfLambda \circ \bfU) \downarrow 0$. Thus there exists some $\delta>0$ such that
\[
\bar\Omega \subseteq \{\bfU \in \cU: \mathbf{S}+\bfLambda \circ \bfU \succeq \delta I\},
\]
which implies that
\[
\bar\Omega = \{\bfU \in \cU:
\log\det(\mathbf{S}+\bfLambda \circ \bfU) ~\ge~ \log\det(\mathbf{S}+\bfLambda \circ \bar \bfU), \ \mathbf{S}+\bfLambda \circ \bfU \succeq \delta I\}.
\]
Hence, $\bar\Omega$ is a compact set. In addition, one can observe that problem~\eqref{dual-prob}
is equivalent to
\[
\max\limits_{\bfU \in \bar\Omega} \log\det(\mathbf{S}+\bfLambda \circ \bfU).
\]
The latter problem clearly has an optimal solution and so is problem~\eqref{dual-prob}.

Finally we show that $\bfX^*=(\mathbf{S}+\bfLambda \circ \bfU^*)^{-1}$ is the unique optimal
solution of \eqref{logdet}, where $\bfU^*$ is an optimal solution of \eqref{dual-prob}.  Since
$\mathbf{S}+\bfLambda \circ \bfU^* \succ 0$, we have $\bfX^* \succ 0$. By the definitions of $\cU$ and $\bfX^*$, and the first-order optimality conditions of~\eqref{dual-prob} at
$\bfU^*$, one can have
\[
\bfU^*_{ij} = \left\{\ba{ll}
1 & \ \mbox{if} \ \bfX^*_{ij} > 0; \\
\beta \in [-1,1] & \ \mbox{if} \ \bfX^*_{ij} = 0; \\
-1 & \ \mbox{otherwise}.
\ea\right.
\]
It follows that $\bfLambda \circ \bfU^* \in \partial  (\sum_{ij}\bfLambda_{ij}  |\bfX_{ij}|)$ at $\bfX=\bfX^*$, where $\partial(\cdot)$ stands for the subdifferential of the associated convex
function.   For convenience, let $f(\bfX)$ denote the objective function of \eqref{logdet}.  Then we have
\[
 -(\bfX^*)^{-1} + \mathbf{S} +\bfLambda \circ \bfU^*  \in \partial f(\bfX^*),
\]
which, together with $\bfX^*=(\mathbf{S}+\bfLambda \circ \bfU^*)^{-1}$, implies that
$0\in \partial f(\bfX^*)$. Hence,
$\bfX^*$ is an optimal solution of \eqref{logdet} and moreover it is unique due to the strict convexity of $-\log\det(\cdot)$.

(b) By statement (a), problem~\eqref{logdet} has a finite optimal value $f^*$. Hence, the above sub-level set $\cL$ is nonempty.  We can observe that for any $\bfX \in \cL$,
\beqa
\frac12 \sum_{ij}\bfLambda_{ij}  |\bfX_{ij}| &=& \ f(\bfX) -  [\underbrace{-\log\det(\mathbf{X}) + \tr(\mathbf{S}\mathbf{X}) + \frac12\sum_{ij}\bfLambda_{ij}  |\bfX_{ij}|}_{\underline f(\bfX)}],  \nn \\
& \le & \alpha - {\underline f}^* , \label{bdd-Xij}
\eeqa
where ${\underline f}^* := \inf\{\underline f(\bfX): \bfX \succ 0\} $. By the assumption $\Diag(\bfS)+\bfLambda >0$, one has $\Diag(\bfS)+\bfLambda/2 >0$. This together with
statement (a) yields ${\underline f}^* \in \Re$.  Notice that $\bfLambda_{ij}>0$ for all $i \neq j$.  This relation and \eqref{bdd-Xij} imply that $\bfX_{ij}$ is bounded for all $\bfX \in \cL$ and $i \neq j$.  In addition, it is well-known that $\det(\bfX) \le \bfX_{11} \bfX_{22} \cdots \bfX_{pp}$ for all $\bfX \succeq 0$. Using this relation, the definition of $f(\cdot)$, and the boundedness of $\bfX_{ij}$ for all  $\bfX \in \cL$ and $i \neq j$, we have that for every $\bfX \in \cL$,
\beqa
\sum_{i} -\log(\bfX_{ii}) + (\bfS_{ii}  + \Lambda_{ii}) \bfX_{ii}  & \le & f(\bfX) - \sum_{i\neq j} (\bfS_{ij} \bfX_{ij} + \Lambda_{ij} |\bfX_{ij}|), \nn \\
& \le & \alpha - \sum_{i\neq j} (\bfS_{ij} \bfX_{ij} + \Lambda_{ij} |\bfX_{ij}|) \ \le \ \delta \label{bdd-Xii}
\eeqa
for some $\delta >0$. In addition, notice from  the assumption that $\bfS_{ii}  + \Lambda_{ii} >0$ for all $i$, and hence
\[
-\log(\bfX_{ii}) + (\bfS_{ii}  + \Lambda_{ii}) \bfX_{ii}  \ \ge \ 1+\min\limits_k \log(\bfS_{kk}  + \Lambda_{kk})  \ =: \ \sigma
\]
for all $i$.  This relation together with \eqref{bdd-Xii}  implies that for every $\bfX \in \cL$ and all $i$,
\[
-\log(\bfX_{ii}) + (\bfS_{ii}  + \Lambda_{ii}) \bfX_{ii} \ \le \ \delta - (p-1) \sigma,
\]
and hence $\bfX_{ii}$ is bounded for all $i$ and $\bfX \in \cL$. We thus conclude that $\cL$ is bounded. In view of this result and the definition of $f$, it is not hard to see that there exists some $\nu >0$ such that $\lambda_{\min}(\bfX) \ge \nu$ for all $\bfX\in \cL$. Hence, one has
\[
\cL =\{\bfX \succeq \nu I:  f(\bfX) \le \alpha\}.
\]
 By the continuity of $f$ on $\{\bfX: \bfX \succeq \nu I\}$,  it follows that $\cL$ is closed. Hence,
$\cL$ is compact.
\end{proof}

\gap

We are now ready to prove Theorem \ref{thm:coord_desc_thm}.

\begin{proof}
Since $\lambda_1 >0$ and $\diag(\mathbf{S}^{(k)}) >0, k=1,\dots, K$, it follows from Lemma~\ref{single-graph} that there exists some $\delta$ such that for each $k=1,\dots,K$,
\[
- \log \det({\mathbf{\Theta}^{(k)}})+\tr(\mathbf{S}^{(k)}{\mathbf{\Theta}^{(k)}}) + \lambda_1 \sum_{i\neq j}|{\mathbf \Theta}_{ij}^{(k)}| ~\ge~ \delta, \quad \forall   {\mathbf{\Theta}^{(k)}} \succ 0.
\]
For convenience, let $h(\mathbf{\Theta})$ denote the objective function of \eqref{eq:fusedloglikelihood} and $\bar{\mathbf{\Theta}}=(\bar{\mathbf{\Theta}}^{(1)},\dots,\bar{\mathbf{\Theta}}^{(K)})$ an arbitrary
feasible point of \eqref{eq:fusedloglikelihood}. Let
\[
\ba{lcl}
\Omega &=& \left\{\mathbf{\Theta}=(\mathbf{\Theta}^{(1)},\dots,\mathbf{\Theta}^{(K)}): h(\mathbf{\Theta}) \le h(\bar{\mathbf{\Theta}}), \ \mathbf{\Theta}^{(k)} \succ 0, k=1,\dots, K\right\}, \\
\Omega_k &=& \left\{\mathbf{\Theta}^{(k)} \succ 0: - \log \det({\mathbf{\Theta}^{(k)}})+\tr(\mathbf{S}^{(k)}{\mathbf{\Theta}^{(k)}}) + \lambda_1 \sum_{i\neq j}|{\mathbf \Theta}_{ij}^{(k)}| \ \le \ \bar\delta\right\}
\ea
\]
for $k=1,\dots, K$, where $\bar\delta=h(\bar{\mathbf{\Theta}})-(K-1)\delta$. Then it is not hard to observe that $\Omega \subseteq \bar \Omega := \Omega_1 \times \cdots \times \Omega_K$. Moreover, problem \eqref{eq:fusedloglikelihood} is equivalent to
\beq \label{equiv-prob}
\min_{\mathbf{\Theta} \in \bar\Omega} h(\mathbf{\Theta}).
\eeq
In view  of Lemma \ref{single-graph}, we know that  $\Omega_k$ is compact for all $k$, which implies that $\bar\Omega$ is also compact. Notice that $h$ is continuous and strictly convex on $\bar\Omega$. Hence, problem \eqref{equiv-prob} has a unique optimal solution and so is problem \eqref{eq:fusedloglikelihood}.
\end{proof}


\section{The screening rule for fused multiple graphical lasso}\label{sec:screening}
Due to the presence of the log determinant, it is challenging to solve the formulations involving the penalized log-likelihood efficiently. The existing methods for single graphical lasso are not scalable to the problems with a large amount of features because of the high computational complexity. Recent studies have shown that the graphical model may contain many connected components, which are disjoint with each other, due to the sparsity of the graphical model, i.e., the corresponding precision matrix has a block diagonal structure (subject to some rearrangement of features). To reduce the computational complexity, it is advantageous to first identify the block structure and then compute the diagonal blocks of the precision matrix instead of the whole matrix.
Danaher et al.~\cite{Danaher2011joint} developed a similar necessary and sufficient condition for fused graphical lasso with two graphs, thus the block structure can be identified. However, it remains a challenge to derive the necessary and sufficient condition for the solution of fused multiple graphical lasso to be block diagonal for $K>2$ graphs.

In this section, we first present a theorem demonstrating that FMGL can be decomposable once its solution has a block diagonal structure. Then we derive a necessary and sufficient condition for the solution of FMGL to be block diagonal for arbitrary number of graphs.

Let $C_1,\dots,C_L$ be a partition of the $p$ features into $L$ non-overlapping sets, with $C_l \cap C_{l'} = \emptyset,~\forall l\neq l'$ and $\bigcup_{l=1}^L C_l=\{1,\dots,p\}$. We say that the solution $\mathbf{\widehat \Theta}$ of FMGL \eqref{eq:fusedloglikelihood} is block diagonal with $L$ known blocks consisting of {features in the sets} $C_l,~l=1,\dots,L$ if there exists a permutation matrix $\mathbf{U} \in \Re^{p \times p} $ such that each estimation precision matrix takes the form of
\beq \label{block-soln}
\widehat{\bf \Theta}^{(k)} ={\bf U}\left( \begin{array}{ccc}\widehat{\bf \Theta}_1^{(k)} &  &\\& \ddots & \\ &  &\widehat{\bf \Theta}_L^{(k)}\end{array}\right){\bf U}^T, \ k = 1,\dots,K.
\eeq
For simplicity of presentation, we assume throughout this paper that $\bf U= I$.

The following decomposition result for problem \eqref{eq:fusedloglikelihood} is straightforward. Its proof is thus omitted.

\begin{theorem}\label{thm:multiBlock}
Suppose that the solution $\mathbf{\widehat \Theta}$ of FMGL \eqref{eq:fusedloglikelihood} is block diagonal with $L$ known $C_l,~l=1,\dots,L$, i.e., each estimated precision matrix has the form~\eqref{block-soln} with $\bf U=I$.
Let ${\mathbf{\widehat \Theta}_l}=(\mathbf{\widehat \Theta}_l^{(1)},\dots,\mathbf{\widehat \Theta}_l^{(K)})$ for $l=1,\dots,L$. Then there holds:
\beq \label{block-FMGL}
\quad\quad {\mathbf{\widehat \Theta}_l} = \arg \min_{\mathbf{\mathbf {\Theta}}_l\succ 0} \sum_{k=1}^K \left(- \log \det({\mathbf{\Theta}_l^{(k)}})+\tr(\mathbf{S}_l^{(k)}{\mathbf{\Theta}_l^{(k)}})\right) + P({\mathbf{\Theta}}_l),~l=1,\dots,L,
\eeq
where ${\mathbf{\Theta}_l^{(k)}}$ and
$\mathbf{S}_l^{(k)}$ are the $|C_l|\times|C_l|$ symmetric submatrices of ${\mathbf{\Theta}^{(k)}}$ and $\mathbf{S}^{(k)}$ corresponding to the $l$-th diagonal block, respectively, for $k=1,\dots,K$, and  $\mathbf{\Theta}_l = (\mathbf{\Theta}_l^{(1)},\dots,\mathbf{\Theta}_l^{(K)})$ for
$l=1,\dots,L$.
\end{theorem}

\gap

The above theorem demonstrates that if a large-scale FMGL problem has a block diagonal solution,
it can then be decomposed into a group of smaller sized FMGL problems. The computational cost for the latter problems can be much cheaper. Now one natural question is how to efficiently identify the block diagonal structure of the FMGL solution before solving the problem. We address this question in the remaining part of this section.

The following theorem provides a necessary and sufficient condition for the solution of FMGL to be
block diagonal with $L$ blocks $C_l, l=1,\dots, L$, which is a key for developing efficient
decomposition scheme for solving FMGL. Since its proof requires some substantial
development of other technical results, we shall postpone the proof until the end of this section.

\begin{theorem}\label{thm:nesssuff}
The FMGL \eqref{eq:fusedloglikelihood} has a block diagonal solution $\mathbf{\widehat \Theta}^{(k)},k=1,\dots,K$ with $L$ known blocks $C_l, l=1,\dots,L$ if and only if $\mathbf{S}^{(k)}, k=1,\dots,K$ satisfy the following inequalities:
\begin{equation}\label{eq:nessuarysuff}
\left\{
\ba{l}
 |\sum^t_{k=1} \bfS_{ij}^{(k)}| \le t \lambda_1 + \lambda_2, \\ [4pt]
 |\sum^{t-1}_{k=0} \bfS_{ij}^{(r+k)}| \le t \lambda_1 + 2\lambda_2, \ 2 \le r \le K-t, \\ [4pt]
 |\sum^{t}_{k=1} \bfS_{ij}^{(K-t+k)}| \le t \lambda_1 + \lambda_2, \\ [4pt]
 |\sum^K_{k=1} \bfS_{ij}^{(k)}| \le K \lambda_1
\ea\right.
\end{equation}
for $t=1,\dots,K-1$, $i\in C_l, j \in C_{l'}, l\neq l'$.
\end{theorem}

\gap

One immediate consequence of Theorem~\ref{thm:nesssuff} is that  the
conditions~\eqref{eq:nessuarysuff} can be used as a screening rule to identify the block
diagonal structure of the FMGL solution. The steps about this rule are described as follows.

\begin{enumerate}
\item Construct an adjacency matrix $\mathbf{E}=\mathbf{I}_{p\times p}$. Set ${\bf E}_{ij}={\bf E}_{ji}=0$ if $\bfS_{ij}^{(k)}, k=1,\dots,K$ satisfy the conditions~\eqref{eq:nessuarysuff}. Otherwise, set ${\bf E}_{ij}={\bf E}_{ji}=1$.
\item Identify the connected components of the adjacency matrix $\mathbf{E}$ (for example, it can be done by calling the Matlab function ``graphconncomp'').
\end{enumerate}

In view of Theorem~\ref{thm:nesssuff}, it is not hard to observe that the resulting connected
components are the partition of the $p$ features into nonoverlapping sets.
It then follows from Theorem~\ref{thm:multiBlock} that a large-scale FMGL problem can be decomposed into a group of smaller sized FMGL problems restricted to the features in each connected component. The computational cost for the
latter problems can be much cheaper. Therefore, this approach may enable us to solve large-scale
FMGL problems very efficiently.


\gap

In the remainder of this section we provide a proof for Theorem~\ref{thm:nesssuff}. Before proceeding,  we establish several technical lemmas as follows.

\gap

\begin{lemma} \label{ext-ray}
Given any two arbitrary index sets $I \subseteq \{1,\cdots,n\}$ and $J \subseteq \{1,\cdots,n-1\}$,
let $\bar I$ and $\bar J$ be the complement of $I$ and $J$ with respect to $\{1,\cdots,n\}$ and $\{1,\cdots,n-1\}$,
respectively. Define
\beq \label{PIJ}
P_{I,J} = \left\{y\in \Re^n: y_I \ge 0, \ y_\bI \le 0, \ y_J - y_{J+1} \ge 0, \ y_\bJ - y_{\bJ + 1} \le 0\right\},
\eeq
where $J+1 = \{j+1: j \in J\}$ and $\bJ + 1 = \{j+1: j \in \bJ\}$. Then, the following statements hold:
\bi
\item[(i)]
Either $P_{I,J}=\{0\}$ or $P_{I,J}$ is unbounded;
\item[(ii)] $0$ is the unique extreme point of $P_{I,J}$;
\item[(iii)]
Suppose that $P_{I,J}$ is unbounded. Then, $\emptyset \neq \ext(P_{I,J}) \subseteq Q$, where $\ext(P_{I,J})$
denotes the set of all extreme rays of $P_{I,J}$, and
\beq \label{d-form}
\ \ \ \ \ \ Q := \{\alpha(\underbrace{0,\cdots,0}_m,\underbrace{1,\cdots,1}_l,0,\cdots,0)^T \in \Re^n: \alpha \neq 0, m \ge 0, 1 \le l \le n\}.
\eeq
\ei
\end{lemma}

\begin{proof}
(i) We observe that $0\in P_{I,J}$. If $P_{I,J} \neq \{0\}$, then there exists $0 \neq y \in P_{I,J}$. Hence,
$\{\alpha y: \alpha \ge 0\} \subseteq P_{I,J}$, which implies that $P_{I,J}$ is unbounded.

(ii) It is easy to see that $0 \in P_{I,J}$ and moreover   there exist $n$ linearly independent
active inequalities at $0$. Hence, $0$ is an extreme point of $P_{I,J}$. On the other hand, suppose
$y$ is an arbitrary extreme point of $P_{I,J}$. Then there exist $n$ linearly independent
active inequalities at $y$, which together with the definition of $P_{I,J}$ immediately implies $y=0$.
Therefore, $0$ is the unique extreme point of $P_{I,J}$.

(iii) Suppose that $P_{I,J}$ is unbounded. By statement (ii), we know that $P_{I,J}$ has a unique extreme point.
Using Minkowski's resolution theorem (e.g., see \cite{bertsekas1997introduction}), we conclude that $\ext(P_{I,J}) \neq
\emptyset$. Let $d \in \ext(P_{I,J})$ be arbitrarily chosen. Then $d \neq 0$. It follows from
\eqref{PIJ} that $d$ satisfies the inequalities
\beq \label{d-ineqs}
d_I \ge 0, \ d_\bI \le 0, \ d_J - d_{J+1} \ge 0, \ d_\bJ - d_{\bJ+1} \le 0,
\eeq
and moreover, the number of independent active inequalities at $d$ is $n-1$. If all entries of $d$ are nonzero, then
$d$ must satisfy $d_J - d_{J+1} = 0$ and $d_\bJ - d_{\bJ+1} = 0$ (with a total number $n-1$), which implies
$d_1=d_2=\cdots=d_n$ and thus $d \in Q$. We now assume that $d$ has at least one zero entry. Then, there exist
positive integers $k$, $\{m_i\}^k_{i=1}$ and $\{n_i\}^k_{i=1}$ satisfying $m_i \le n_i < m_{i+1} \le n_{i+1}$
for $i=1,\ldots,k-1$ such that
\beq \label{zero-position}
\{i: d_i=0\} = \{m_1, \cdots, n_1\} \cup \{m_2, \cdots, n_2\} \cup \cdots \cup \{m_k,\cdots,n_k\}.
\eeq
One can immediately observe that
\beq \label{partial-eqs}
d_{m_i} = \cdots = d_{n_i} = 0, \quad d_j - d_{j+1} = 0, \quad  m_i \le j \le n_i-1, \ 1 \le i \le k.
\eeq
We next divide the rest of proof into four cases.

Case (a): $m_1=1$ and $n_k=n$. In view of \eqref{zero-position}, one can observe that $d_{m_i-1}-d_{m_i} \neq 0$ and
$d_{n_{i-1}}-d_{n_{i-1}+1} \neq 0$ for $i=2,\ldots,k$. We then see from \eqref{d-ineqs} that except the active
inequalities given in \eqref{partial-eqs}, all other possible active inequalities at $d$ are
\beq \label{other-ineqs-a}
d_j-d_{j+1}=0, \quad  n_{i-1} < j < m_{i}-1, \  2 \le i \le k
\eeq
(with a total number $\sum^k_{i=2} (m_i-n_{i-1}-2)$). Notice that the total number of independent active
inequalities given in \eqref{partial-eqs} is $\sum^k_{i=1} (n_i-m_i+1)$. Hence, the number of independent active
inequalities at $d$ is at most
\[
\sum^k_{i=1} (n_i-m_i+1) + \sum^k_{i=2} (m_i-n_{i-1}-2) = n_k - m_1 - k + 2 = n-k+1.
\]
Recall that the number of independent active inequalities at $d$ is $n-1$. Hence, we have
$n-k+1 \ge n-1$, which implies $k \le 2$. Due to $d \neq 0$, we observe that $k \neq 1$ holds for this
case. Also, we know that $k > 0$. Hence, $k=2$. We then see that all possible active inequalities described in
\eqref{other-ineqs-a} must be active at $d$, which together with $k=2$ immediately implies that $d \in Q$.

Case (b): $m_1=1$ and $n_k < n$. Using \eqref{zero-position}, we observe that $d_{m_i-1}-d_{m_i} \neq 0$
for $i=2,\ldots,k$ and $d_{n_i}-d_{n_i+1} \neq 0$ for $i=1,\ldots,k$. In view of these relations and a similar
argument as in case (a), one can see that the number of independent active inequalities at $d$ is at most
\[
\sum^k_{i=1} (n_i-m_i+1) + \sum^k_{i=2} (m_i-n_{i-1}-2) + n-n_k-1 = n - m_1 - k + 1 = n-k.
\]
Similarly as in case (a), we can conclude from the above relation that $k = 1$ and $d \in Q$.

Case (c): $m_1>1$ and $n_k = n$. By \eqref{zero-position}, one can observe that $d_{m_i-1}-d_{m_i} \neq 0$
for $i=1,\ldots,k$ and $d_{n_i}-d_{n_i+1} \neq 0$ for $i=1,\ldots,k-1$. Using these relations and a similar
argument as in case (a), we see that the number of independent active inequalities at $d$ is at most
\[
m_1-2 + \sum^k_{i=1} (n_i-m_i+1) + \sum^k_{i=2} (m_i-n_{i-1}-2) = n_k  - k = n-k.
\]
Similarly as in case (a), we can conclude from the above relation that $k = 1$ and $d \in Q$.

Case (d): $m_1>1$ and $n_k < n$. From \eqref{zero-position}, one can observe that $d_{m_i-1}-d_{m_i} \neq 0$
for $i=1,\ldots,k$ and $d_{n_i}-d_{n_i+1} \neq 0$ for $i=1,\ldots,k$. By virtue of these relations and a similar
argument as in case (a), one can see that the number of independent active inequalities at $d$ is at most
\[
m_1-2+\sum^k_{i=1} (n_i-m_i+1) + \sum^k_{i=2} (m_i-n_{i-1}-2) + n-n_k-1 = n-k-1.
\]
Recall that $k \ge 1$ and the number of independent active inequalities at $d$ is $n-1$. Hence, this case cannot occur.

Combining the above four cases, we conclude that $\ext(P_{I,J}) \subseteq Q$.
\end{proof}

\gap

\begin{lemma} \label{extray-repr}
Let $P_{IJ}$ and $Q$ be defined in \eqref{PIJ} and \eqref{d-form}, respectively. Then,
\[
\cup\left\{\ext(P_{I,J}): I \subseteq \{1,\cdots,n\}, \ J \subseteq \{1,\cdots,n-1\}\right\} = Q.
\]
\end{lemma}

\begin{proof}
It follows from Lemma \ref{ext-ray} (iii) that
\[
\cup\left\{\ext(P_{I,J}): I \subseteq \{1,\cdots,n\}, \ J \subseteq \{1,\cdots,n-1\}\right\} \ \subseteq \ Q.
\]
We next show that
\[
\cup\left\{\ext(P_{I,J}): I \subseteq \{1,\cdots,n\}, \ J \subseteq \{1,\cdots,n-1\}\right\} \ \supseteq \ Q.
\]
Indeed, let $d\in Q$ be arbitrarily chosen. Then, there exist $\alpha \neq 0$ and positive integers $m_1$
and $n_1$ satisfying $1 \le m_1 \le n_1$ such that $d_i=\alpha$ for $m_1 \le i \le n_1$ and the rest of
$d_i$'s are 0. If $\alpha>0$, it is not hard to see that $d\in \ext(P_{I,J})$ with $I=\{1,\cdots,n\}$ and
$J=\{m_1,\cdots,n-1\}$. Similarly, if $\alpha<0$, $d\in \ext(P_{I,J})$ with $I=\emptyset$ and $J$ being the complement of $\bar{J} = \{m_1,\cdots,n-1\}$. Hence,
$d \in \cup\left\{\ext(P_{I,J}): I \subseteq \{1,\cdots,n\}, \ J \subseteq \{1,\cdots,n-1\}\right\}$.
\end{proof}

\gap

\begin{lemma} \label{repr}
Let $x\in \Re^n$, $\lambda_1$, $\lambda_2 \ge 0$ be given, and let
\[
f(y):= x^Ty - \lambda_1 \sum^n_{i=1}|y_i|-\lambda_2\sum^{n-1}_{i=1}|y_i-y_{i+1}|.
\]
Then, $f(y) \le 0$ for all $y \in \Re^n$ if and only if
$x$ satisfies the following inequalities:
\[
\left\{
\ba{l}
 |\sum^k_{j=1} x_j| \le k \lambda_1 + \lambda_2, \\ [4pt]
 |\sum^{k-1}_{j=0} x_{i+j}| \le k \lambda_1 + 2\lambda_2, \ 2 \le i \le n-k, \\ [4pt]
 |\sum^{k}_{j=1} x_{n-k+j}| \le k \lambda_1 + \lambda_2, \\ [4pt]
 |\sum^n_{j=1} x_j| \le n \lambda_1
\ea\right.
\]
for $k=1,\ldots, n-1$.
\end{lemma}

\begin{proof}
Let $P_{I,J}$ be defined in \eqref{PIJ} for any  $I \subseteq \{1,\ldots,n\}$ and
$J \subseteq \{1,\ldots,n-1\}$. We observe that
\bi
\item[(a)]
$\Re^n = \cup\left\{P_{I,J}: \ I \subseteq \{1,\ldots,n\}, \ J \subseteq \{1,\ldots,n-1\}\right\}$;
\item[(b)] $f(y) \le 0$ for all $y\in \Re^n$ if and only if $f(y) \le 0$ for all $y\in P_{I,J}$, and every
$I \subseteq \{1,\ldots,n\}$ and $J \subseteq \{1,\ldots,n-1\}$;
\item[(c)] $f(y)$ is a linear function of $y$ when restricted to the set $P_{I,J}$ for every $I \subseteq \{1,\ldots,n\}$ and $J \subseteq \{1,\ldots,n-1\}$.
\ei
If $P_{I,J}$ is bounded, we have $P_{I,J}=\{0\}$ and $f(y)=0$ for $y \in P_{I,J}$. Suppose
that $P_{I,J}$ is unbounded. By Lemma \ref{ext-ray} and Minkowski's resolution theorem,
$P_{I,J}$ equals the finitely generated cone by $\ext(P_{I,J})$. It then follows that $f(y) \le 0$ for all
$y\in P_{I,J}$ if and only if $f(d) \le 0$ for all $d\in\ext(P_{I,J})$. Using these facts and
Lemma \ref{extray-repr}, we see that $f(y) \le 0$ for all $y\in \Re^n$ if and only if $f(d) \le 0$
for all $d\in Q$, where $Q$ is defined in \eqref{d-form}. By the definitions of $Q$ and $f$, we further
observe that $f(y) \le 0$ for all $y\in \Re^n$ if and only if $f(d) \le 0$ for all
\[
d \in \left\{\pm(\underbrace{0,\cdots,0}_m,\underbrace{1,\cdots,1}_l,0,\cdots,0)^T \in \Re^n:  m \ge 0, 1 \le l \le n\right\},
\]
which together with the definition of $f$ immediately implies that the conclusion of this lemma holds.
\end{proof}

\gap

\begin{lemma}\label{thm:linearsy}
Let $x\in \Re^n$, $\lambda_1$, $\lambda_2 \ge 0$ be given. The linear system
\beq \label{lin-sys}
\left\{
\ba{l}
x_1 + \lambda_1 \gamma_1 + \lambda_2 v_1 = 0, \\ [4pt]
x_i + \lambda_1 \gamma_i + \lambda_2 (v_i-v_{i-1}) = 0,  \ 2 \le i \le n-1,  \\[ 4pt]
x_n  + \lambda_1 \gamma_n - \lambda_2 v_{n-1}  = 0, \\ [4pt]
-1 \le \gamma_i \le 1, \ i =1, \ldots, n,  \\ [4pt]
-1 \le v_i \le 1, \ i =1, \ldots, n-1 \\ [4pt]
\ea
\right.
\eeq
has a solution $(\gamma, v)$ if and only if $(x,\lambda_1,\lambda_2)$ satisfies the
following inequalities:
\[
\left\{
\ba{l}
 |\sum^k_{j=1} x_j| \le k \lambda_1 + \lambda_2, \\ [4pt]
 |\sum^{k-1}_{j=0} x_{i+j}| \le k \lambda_1 + 2\lambda_2, \ 2 \le i \le n-k, \\ [4pt]
 |\sum^{k}_{j=1} x_{n-k+j}| \le k \lambda_1 + \lambda_2, \\ [4pt]
 |\sum^n_{j=1} x_j| \le n \lambda_1
\ea\right.
\]
for $k=1,\ldots, n-1$.
\end{lemma}

\begin{proof}
The linear system \eqref{lin-sys} has a solution if and only if the linear programming
\beq \label{primal}
\min\limits_{\gamma,v} \{0^T \gamma + 0^T v: (\gamma,v) \ \mbox{satisfies} \ \eqref{lin-sys}\}
\eeq
has an optimal solution. The Lagrangian dual of \eqref{primal} is
\[
\max\limits_y \min\limits_{\gamma,v}\left\{x^T y + \lambda_1 \sum^n_{i=1} y_i \gamma_i +
\lambda_2\sum^{n-1}_{i=1} (y_i-y_{i+1}) v_i: \ -1 \le \gamma, v \le 1\right\},
\]
which is equivalent to
\beq \label{dual}
\max\limits_{y} f(y) := x^Ty - \lambda_1 \sum^n_{i=1}|y_i|-\lambda_2\sum^{n-1}_{i=1}|y_i-y_{i+1}|.
\eeq
By the Lagrangian duality theory, problem \eqref{primal} has an optimal solution if and only if its
dual problem \eqref{dual} has optimal value $0$, which is equivalent to $f(y) \le 0$ for all
$y\in\Re^n$. The conclusion of this lemma then immediately follows from Lemma \ref{repr}.
\end{proof}

\gap

We are now ready to prove Theorem~\ref{thm:nesssuff}.

\gap

\begin{proof}
For the sake of convenience, we denote the inverse of  $\mathbf{\widehat \Theta}^{(k)}$ as
$\mathbf{\widehat W}^{(k)}$ for $k=1,\dots, K$. By  the first-order optimality
 conditions, we observe that $\mathbf{\widehat \Theta}^{(k)} \succ 0, k=1,\dots, K$ is the optimal solution of problem~\eqref{eq:fusedloglikelihood} if and only if it satisfies
\begin{eqnarray}
&&\quad\quad -\hw^{(k)}_{ii}+\bfS_{ii}^{(k)} = 0,~ 1\leq k \leq K, \label{opt-cond1} \\
&&\quad\quad -\hw^{(1)}_{ij}+\bfS_{ij}^{(1)}  + \lambda_1\gamma_{ij}^{(1)}+ \lambda_2 \upsilon_{ij}^{(1,2)} = 0 ,
\label{opt-cond2} \\
&&\quad\quad -\hw^{(k)}_{ij}+\bfS_{ij}^{(k)}  + \lambda_1\gamma_{ij}^{(k)}+ \lambda_2(-\upsilon_{ij}^{(k-1,k)} + \upsilon_{ij}^{(k,k+1)}) = 0, \ 2\leq k\leq K-1, \label{opt-cond3} \\
&&\quad\quad -\hw^{(K)}_{ij}+\bfS_{ij}^{(K)}  + \lambda_1\gamma_{ij}^{(K)}- \lambda_2 \upsilon_{ij}^{(K-1,K)} = 0 \label{opt-cond4}
\end{eqnarray}
for all $i,j=1,\dots,p, i \neq j$,
where $\gamma_{ij}^{(k)}$ is a subgradient of $|\bfTheta_{ij}^{(k)}|$ at
$\bfTheta_{ij}^{(k)}=\htheta_{ij}^{(k)}$; and  $\upsilon_{ij}^{(k,k+1)}$ is a subgradient of $|\bfTheta_{ij}^{(k)}
- \bfTheta_{ij}^{(k+1)}|$ with respect to $\bfTheta_{ij}^{(k)}$ at $(\bfTheta_{ij}^{(k)},
\bfTheta_{ij}^{(k+1)})=(\htheta_{ij}^{(k)}, \htheta_{ij}^{(k+1)})$, that is, $\upsilon_{ij}^{(k,k+1)} = 1$ if $\htheta_{ij}^{(k)} > \htheta_{ij}^{(k+1)}$, $\upsilon_{ij}^{(k,k+1)} = -1$ if $\htheta_{ij}^{(k)} < \htheta_{ij}^{(k+1)}$, and $\upsilon_{ij}^{(k,k+1)} \in [-1,1]$ if $\htheta_{ij}^{(k)} = \htheta_{ij}^{(k+1)}$.

{\bf Necessity: } Suppose that $\mathbf{\widehat \Theta}^{(k)},k=1,\dots,K$ is a block diagonal
optimal solution of problem~\eqref{eq:fusedloglikelihood} with $L$ known blocks
$C_l, l=1,\dots,L$. Note that $\mathbf{\widehat W}^{(k)}$ has the same block diagonal structure
as $\mathbf{\widehat \Theta}^{(k)}$. Hence, $\hw^{(k)}_{ij} =\htheta_{ij}^{(k)} = 0$ for
$i\in C_l, j \in C_{l'}, l\neq l'$. This together with~\eqref{opt-cond2}-\eqref{opt-cond4} implies that
for each $i\in C_l, j \in C_{l'}, l\neq l'$, there exist $(\gamma^{(k)}_{ij},v^{(k,k+1)}_{ij}), k=1,\dots,K-1$ and $\gamma^{(K)}_{ij}$ such that
\beq \label{eq:KKTforfusedwithoutW}
\begin{array}{l}
\bfS_{ij}^{(1)}  + \lambda_1\gamma_{ij}^{(1)}+ \lambda_2 \upsilon_{ij}^{(1,2)} = 0,\\ [4pt]
\bfS_{ij}^{(k)}  + \lambda_1\gamma_{ij}^{(k)}+ \lambda_2(-\upsilon_{ij}^{(k-1,k)} + \upsilon_{ij}^{(k,k+1)}) = 0, \ 2\leq k\leq K-1,\\ [4pt]
\bfS_{ij}^{(K)}  + \lambda_1\gamma_{ij}^{(K)}- \lambda_2 \upsilon_{ij}^{(K-1,K)} = 0, \\ [4pt]
-1 \le\gamma^{(k)}_{ij} \le 1, \ 1\leq k\leq K,  \\ [4pt]
-1 \le v^{(k,k+1)}_{ij} \le 1, \  \ 1\leq k\leq K-1. \\ [4pt]
\end{array}
\eeq
Using \eqref{eq:KKTforfusedwithoutW} and Lemma~\ref{thm:linearsy},  we see that
\eqref{eq:nessuarysuff} holds for $t=1,\dots,K-1$, $i\in C_l, j \in C_{l'}, l\neq l'$.

{\bf Sufficiency:}
Suppose that \eqref{eq:nessuarysuff} holds for $t=1,\dots,K-1$, $i\in C_l, j \in C_{l'}, l\neq l'$.
It then follows from  Lemma~\ref{thm:linearsy} that for each $i\in C_l, j \in C_{l'}, l\neq l'$, there
exist $(\gamma^{(k)}_{ij},v^{(k,k+1)}_{ij}), k=1,\dots,K-1$ and $\gamma^{(K)}_{ij}$ such that
\eqref{eq:KKTforfusedwithoutW} holds. Now let $\mathbf{\widehat \Theta}^{(k)},k=1,\dots,K$ be a
block diagonal matrix as defined in \eqref{block-soln} with $\bf U=I$, where ${\mathbf{\widehat \Theta}_l}=(\mathbf{\widehat \Theta}_l^{(1)},\dots,\mathbf{\widehat \Theta}_l^{(K)})$  is given by
\eqref{block-FMGL} for $l=1,\dots,L$. Also, let $\mathbf{\widehat W}^{(k)}$ be the inverse of
$\mathbf{\widehat \Theta}^{(k)}$ for $k=1,\dots, K$. Since ${\mathbf{\widehat \Theta}_l}$ is the
optimal solution of problem~\eqref{block-FMGL}, the first-order optimality conditions imply
that~\eqref{opt-cond1}-\eqref{opt-cond4} hold for all $i,j\in  C_l, i\neq j, l=1,\dots,L$. Notice that
 $\htheta_{ij}^{(k)} = \hw^{(k)}_{ij} = 0$ for every $i\in C_l, j \in C_{l'}, l\neq l'$. Using this fact
and \eqref{eq:KKTforfusedwithoutW}, we observe that~\eqref{opt-cond1}-\eqref{opt-cond4} also
hold for all $i\in C_l, j \in C_{l'}, l\neq l'$.  It then follows that $\mathbf{\widehat \Theta}^{(k)},k=1,\dots,K$ is an optimal solution of problem~\eqref{eq:fusedloglikelihood}.  In addition, $\mathbf{\widehat \Theta}^{(k)},k=1,\dots,K$  is block diagonal with $L$ known blocks $C_l, l=1,\dots,L$. The conclusion thus holds.
\end{proof}

\section{Second-order method}\label{sec:Newtonmethod}
The screening rule proposed in Section \ref{sec:screening} is capable of partitioning all features
into a group of smaller sized blocks. Accordingly, a large-scale FMGL \eqref{eq:fusedloglikelihood}
can be decomposed into a number of smaller sized FMGL problems.  For each block $l$, we
need to compute its individual estimated  precision matrix $\mathbf{\Theta}_l^{(k)}$ by solving the
FMGL \eqref{eq:fusedloglikelihood} with $\mathbf{S}^{(k)}$ replaced by $\mathbf{S}_l^{(k)}$. In this
section, we discuss how to solve those single block FMGL problems efficiently. For simplicity of
presentation, we assume throughout this section that the FMGL \eqref{eq:fusedloglikelihood}
has only one block, that is, $L=1$.

We now propose a second-order method to solve the FMGL \eqref{eq:fusedloglikelihood}. For simplicity of notation, we let $\mathbf{\Theta} :=(\mathbf{\Theta}^{(1)},\dots,\mathbf{\Theta}^{(K)})$ and use $t$ to denote the Newton iteration index. Let $\mathbf{\Theta}_t = (\mathbf{\Theta}^{(1)}_t,\dots,\mathbf{\Theta}^{(K)}_t)$ be the approximate solution obtained at the $t$-th Newton iteration.

The optimization problem  \eqref{eq:fusedloglikelihood} can be rewritten as
\begin{equation}
  \label{eq:fusedloglikelihood_concise}
 \min_{\mathbf{\Theta}\succ 0} F(\mathbf{\Theta}) := \sum_{k=1}^K f_k(\mathbf{\Theta}^{(k)}) + P(\mathbf{\Theta}),
\end{equation}
where
$$
f_k(\mathbf{\Theta}^{(k)}) =  - \log \det({\mathbf{\Theta}^{(k)}})+\tr(\mathbf{S}^{(k)}{\mathbf{\Theta}^{(k)}}).
$$
In the second-order method, we approximate the objective function $F(\mathbf{\Theta})$ at the current iterate ${\mathbf{\Theta}_{t}}$ by a ``quadratic'' model $Q_t(\mathbf{\Theta})$:
\begin{equation}
  \label{eq:fmgl_qp_app}
\min_{\mathbf{\Theta}} Q_t(\mathbf{\Theta}) :=\sum_{k=1}^K q_k(\mathbf{\Theta}^{(k)}) + P(\mathbf{\Theta}),
\end{equation}
where $q_k$ is the quadratic approximation of $f_k$ at $\mathbf{\Theta}^{(k)}_t$, that is,
\[
\quad\quad q_k(\mathbf{\Theta}^{(k)})= \frac{1}{2}\tr(\mathbf{W}^{(k)}_t\mathbf{D}^{(k)}\mathbf{W}^{(k)}_t\mathbf{D}^{(k)}) + \tr((\mathbf{S}^{(k)} - \mathbf{W}^{(k)}_t)\mathbf{D}^{(k)}) + f_k(\mathbf{\Theta}^{(k)}_t)
\]
with $\mathbf{W}^{(k)}_t= (\mathbf{\Theta}^{(k)}_t)^{-1}$ and  $\mathbf{D}^{(k)} = \mathbf{\Theta}^{(k)} - \mathbf{\Theta}^{(k)}_t$. Suppose that $\bar{\mathbf{\Theta}}_{t+1}$ is the optimal solution of \eqref{eq:fmgl_qp_app}. Then we obtain the Newton search direction
\begin{equation}
  \label{eq:fmgl_newton_d}
  \mathbf{D} = \bar{\mathbf{\Theta}}_{t+1} - \mathbf{\Theta}_{t}.
\end{equation}

We shall mention that the subproblem \eqref{eq:fmgl_qp_app} can be suitably solved by the non-monotone spectral projected gradient (NSPG) method (see, for example,
\cite{wright2009sparse,lu2011augmented}).  It was shown by Lu and Zhang \cite{lu2011augmented} that the NSPG method is locally linearly convergent. Numerous computational studies have demonstrated that the NSPG method is very efficient though its global convergence rate is so far unknown. When applied to \eqref{eq:fmgl_qp_app}, the NSPG method  requires solving the proximal  subproblems in the form of
\begin{equation}
  \label{eq:fmgl_newton_prox}
\min_{\mathbf{\Theta}} \frac{1}{2}\sum_{k=1}^K \|\mathbf{\Theta}^{(k)} - \mathbf{G}^{(k)}\|^2_F + \alpha P(\mathbf{\Theta})
\end{equation}
for some $\mathbf{G} = (\mathbf{G}^{(1)},\dots,\mathbf{G}^{(K)})$ and $\alpha >0$.
By the definition of $P(\mathbf{\Theta})$, it is not hard to see that problem \eqref{eq:fmgl_newton_prox} can be decomposed into a set of independent and smaller sized problems
\begin{equation}
  \label{eq:fusedlasso}
\quad\quad\min_{\bfTheta^{(k)}_{ij},k=1,\dots,K} \frac{1}{2}\sum_{k=1}^K (\bfTheta^{(k)}_{ij} - \mathbf{G}^{(k)}_{ij})^2 + \alpha_1 \sum_{k=1}^K|\bfTheta_{ij}^{(k)}| + \alpha_2 \sum_{k=1}^{K-1}|\bfTheta_{ij}^{(k)} - \bfTheta_{ij}^{(k+1)}|
\end{equation}
for all $ i \ge j, \ j=1,\dots,p$, where $(\alpha_1,\alpha_2) = \alpha (\lambda_1,\lambda_2)$.
The problem~\eqref{eq:fusedlasso} is known as the fused lasso signal approximator, which can be solved very efficiently and exactly~\cite{condat2012direct,liu2010efficient}. 
 In addition, they are independent from each other and thus can be solved in parallel.

Given the current search direction $\mathbf{D}=(\mathbf{D}^{(1)},\dots,\mathbf{D}^{(K)})$
that is computed above, we need to find the suitable step length $\beta\in(0,1]$ to ensure a sufficient reduction in the objective function of \eqref{eq:fusedloglikelihood}  and positive definiteness of the next iterate
$\mathbf{\Theta}^{(k)}_{t+1} = \mathbf{\Theta}^{(k)}_t + \beta\mathbf{D}^{(k)},k=1,\dots,K$.
In the context of the standard (single) graphical lasso, Hsieh et al.~\cite{hsieh2011sparse} have shown that a step length satisfying the above requirements always exists. We can similarly prove
that the desired step length also exists for the FMGL \eqref{eq:fusedloglikelihood}.

\begin{lemma}\label{thm:linesearch}
Let $\mathbf{\Theta}_t = (\mathbf{\Theta}^{(1)}_t,\dots,\mathbf{\Theta}^{(K)}_t)$ be such that
${\mathbf{\Theta}}^{(k)}_t\succ0$ for $k=1,\dots, K$, and let $\mathbf{D}=(\mathbf{D}^{(1)},\dots,\mathbf{D}^{(K)})$ be the associated Newton search direction computed according to \eqref{eq:fmgl_qp_app}. Suppose $\mathbf{D} \neq 0$.\footnote {It is well known that if $\mathbf{D}=0$, $\mathbf{\Theta}_t$ is the optimal solution of problem \eqref{eq:fusedloglikelihood}.} Then there exists a $\bar \beta  >0$ such that ${\mathbf{\Theta}}^{(k)}_t+\beta {\mathbf{D}}^{(k)}\succ0$ and the sufficient reduction condition
\begin{equation}
  \label{eq:linesearch}
  F(\mathbf{\Theta}_{t} + \beta\mathbf{D})\leq F(\mathbf{\Theta}_{t}) + \sigma\beta\delta
\end{equation}
holds for all $0<\beta < \bar \beta$, where $\sigma\in(0,1)$ is a given constant and
\[
\delta =\sum_{k=1}^K \tr((\mathbf{S}^{(k)} - \mathbf{W}^{(k)}_t)\mathbf{D}^{(k)}) +P(\mathbf{\Theta}_{t} + \mathbf{D}) - P(\mathbf{\Theta}_{t}).
\]
\end{lemma}

\begin{proof}
Let $\tilde \beta= 1/\max\{\|({\mathbf{\Theta}}^{(k)}_t)^{-1}{\mathbf{D}}^{(k)}\|_2: k=1,\dots,K\}$, where $\|\cdot\|_2$ denotes the spectral norm of a matrix.
Since $\mathbf{D} \neq 0$ and ${\mathbf{\Theta}}^{(k)}_t\succ0, k=1,\dots, K$, we see that $\tilde \beta >0$. Moreover, we have for all $0<\beta < \tilde \beta$ and $k=1,\dots, K$,
\[
\begin{array}{lcl}
({\mathbf{\Theta}}^{(k)}_t)^{-\frac12}\left({\mathbf{\Theta}}^{(k)}_t+\beta {\mathbf{D}}^{(k)}\right) ({\mathbf{\Theta}}^{(k)}_t)^{-\frac12} &=& {\mathbf I} + \beta ({\mathbf{\Theta}}^{(k)}_t)^{-\frac12}{\mathbf{D}}^{(k)}({\mathbf{\Theta}}^{(k)}_t)^{-\frac12} \\ [4pt]
& \succeq & (1-\beta \|({\mathbf{\Theta}}^{(k)}_t)^{-1}{\mathbf{D}}^{(k)}\|_2) {\mathbf I} \ \succ 0.
\end{array}
\]
By the definition of $\mathbf D$ and \eqref{eq:fmgl_qp_app}, one can easily show that
\[
\delta\ \leq \ -\sum_{k=1}^K\tr(\mathbf{W}^{(k)}_t\mathbf{D}^{(k)}\mathbf{W}^{(k)}_t\mathbf{D}^{(k)}),
\]
which together with the fact that $\mathbf{W}^{(k)}_t \succ 0, k=1,\dots, K$ and ${\mathbf D} \neq 0$ implies that $\delta <0$. Using differentiability of $f_k$, convexity of $P$, and the definition of $\delta$, we obtain that for all sufficiently small $\beta>0$,
\[
\ba{l}
F(\mathbf{\Theta}_{t} + \beta\mathbf{D})- F(\mathbf{\Theta}_{t}) \ =\ \sum_{k=1}^K(f_{k}(\mathbf{\Theta}^{(k)}_t+\beta\mathbf{D}^{(k)})-f_k(\mathbf{\Theta}^{(k)}_t))
+ P(\mathbf{\Theta}_{t} + \beta\mathbf{D}) - P(\mathbf{\Theta}_{t}), \\ [4pt]
= \sum_{k=1}^K\tr((\mathbf{S}^{(k)} - \mathbf{W}^{(k)}_t)\mathbf{D}^{(k)}) \beta+ o(\beta) +
P(\beta(\mathbf{\Theta}_{t} + \mathbf{D})+(1-\beta)\mathbf{\Theta}_{t}) - P(\mathbf{\Theta}_{t}), \\ [4pt]
\le \sum_{k=1}^K\tr((\mathbf{S}^{(k)} - \mathbf{W}^{(k)}_t)\mathbf{D}^{(k)}) \beta+ o(\beta) +
\beta P(\mathbf{\Theta}_{t} + \mathbf{D})+(1-\beta)P(\mathbf{\Theta}_{t}) - P(\mathbf{\Theta}_{t}), \\ [4pt]
\leq \beta \delta + o(\beta).
\ea
\]
This inequality together with $\delta<0$ and $\sigma \in (0,1)$ implies that there exists $\hat \beta >0$ such that
for all $\beta \in (0,\hat\beta)$, $F(\mathbf{\Theta}_{t} + \beta\mathbf{D})- F(\mathbf{\Theta}_{t})  \le  \sigma \beta \delta$. It then follows that the conclusion of this lemma holds for $\bar\beta = \min\{\tilde \beta, \hat\beta\}$.
\end{proof}

\gap

By virtue of Lemma~\ref{thm:linesearch}, we can adopt the well-known Armijo's backtracking line search rule~\cite{tseng2009coordinate} to select a step length $\beta \in (0,1]$ so that
${\mathbf{\Theta}}^{(k)}_t+\beta {\mathbf{D}}^{(k)}\succ0$ and
 \eqref{eq:linesearch} holds. In particular,  we choose $\beta$ to be the largest number of the sequence $\{1,1/2,\dots,1/2^i,\dots\}$ that satisfies these requirements.
We can use the Cholesky factorization to check the positive definiteness of
$\mathbf{\Theta}^{(k)}_{t} + \beta\mathbf{D}^{(k)},k=1,\dots,K$~\cite{hsieh2011sparse}. In addition, the associated terms  $\log\det(\mathbf{\Theta}^{(k)}_{t} + \beta\mathbf{D}^{(k)})$ and $(\mathbf{\Theta}^{(k)}_{t} + \beta\mathbf{D}^{(k)})^{-1}$ can be efficiently computed as a byproduct of the Cholesky decomposition of $\mathbf{\Theta}^{(k)}_{t} + \beta\mathbf{D}^{(k)}$.

\subsection{Shrinking scheme}
Given the large number of unknown variables in \eqref{eq:fmgl_qp_app}, it is advantageous to minimize~\eqref{eq:fmgl_qp_app} in a reduced space. The issue now is how to identify the reduced space. In the case of a single graph ($K=1$), problem \eqref{eq:fmgl_qp_app} degenerates to a lasso problem of size $p^2$. Hsieh et al.~\cite{hsieh2011sparse} proposed a strategy to determine a subset of variables that are allowed to be updated in each Newton iteration for single graphical lasso. Specifically, the $p^2$ variables in single graphical lasso are partitioned into two sets, $J_{free}$ and $J_{fixed}$, based on the gradient at the start of each Newton iteration, and then the minimization is only performed on the variables in $J_{free}$. We call this technique ``shrinking'' in this paper. Due to the sparsity of the precision matrix, the size of $J_{free}$ is usually much smaller than $p^2$. Moreover, it has been shown in the single graph case that the size of $J_{free}$ will decrease quickly~\cite{hsieh2011sparse}. The shrinking technique can thus improve the computational efficiency. This technique was also successfully used in \cite{joachims1999making,olsen2012newton,yuan2012improved}. We show that shrinking can be extended to the fused multiple graphical lasso based on the results established in Section~\ref{sec:screening}.

Denote the gradient of $f_k$ at $t$-th iteration by $\widetilde {\mathbf{G}}_t^{(k)} = \mathbf{S}^{(k)} - \mathbf{W}_t^{(k)}$, {and its $(i,j)$-th element by $\widetilde {\mathbf{G}}_{t,ij}^{(k)}$.} Then we have the following result.

\gap

\begin{lemma}
\label{thm:fgml_lscreening}
For ${\mathbf \Theta}_t$ in the $t$-th iteration, define the fixed set $J_{fixed}$ as
\[
\ba{l}
J_{fixed} = \{(i,j)|\bfTheta_{t,ij}^{(1)} = \dots =\bfTheta_{t,ij}^{(K)}=0 \mbox{ and } \widetilde {\mathbf{G}}_{t,ij}^{(1)},\dots,\widetilde {\mathbf{G}}_{t,ij}^{(K)} \mbox{ satisfy the inequalities} \\ ~~~~~~~~~~\mbox{ below}\}.
\ea
\]
\begin{equation}
  \label{eq:localscreening}
  \left\{
\ba{l}
 |\sum^u_{k=1} \widetilde {\mathbf{G}}_{t,ij}^{(k)}| < u \lambda_1 + \lambda_2, \\ [4pt]
 |\sum^{u-1}_{k=0} \widetilde {\mathbf{G}}_{t,ij}^{(r+k)}| < u \lambda_1 + 2\lambda_2, \ 2 \le r \le K-u, \\ [4pt]
 |\sum^{u}_{k=1} \widetilde {\mathbf{G}}_{t,ij}^{(K-u+k)}| < u \lambda_1 + \lambda_2, \\ [4pt]
 |\sum^K_{k=1} \widetilde {\mathbf{G}}_{t,ij}^{(k)}| < K \lambda_1
\ea\right.
\end{equation}
for $u=1,\dots, K-1.$

Then, the solution of the following optimization problem is $\mathbf{D}^{(1)}=\dots=\mathbf{D}^{(K)}=0:$
\begin{equation}
  \label{eq:fmgl_qp_app_const}
  \min_{\mathbf{D}} Q_t(\mathbf{\Theta}_t+\mathbf{D}) \mbox{ such that } \mathbf{D}_{ij}^{(1)}=\dots=\mathbf{D}_{ij}^{(K)}=0,(i,j)\notin J_{fixed}.
\end{equation}
\end{lemma}
\begin{proof}
Consider problem \eqref{eq:fmgl_qp_app_const}, which can be reformulated to
\begin{equation}
  \label{eq:fmgl_qp_app2}
  \begin{array}{ll}
   \min_{\mathbf{D}} & \sum_{k=1}^K\left(\frac{1}{2}\vec(\mathbf{D}^{(k)})^T\mathbf{H}_t^{(k)}\vec(\mathbf{D}^{(k)}) + \vec(\widetilde {\mathbf{G}}_t^{(k)})^T\vec(\mathbf{D}^{(k)})\right)\\
   & + P(\mathbf{\Theta}_t+\mathbf{D}),\\
   &s.t.~~~\mathbf{D}_{ij}^{(1)}=\dots=\mathbf{D}_{ij}^{(K)}=0,(i,j)\notin J_{fixed},
   \end{array}
\end{equation}
where $\mathbf{H}_t^{(k)} = \mathbf{W}_t^{(k)}\otimes\mathbf{W}_t^{(k)}$. Because of the constraint $\mathbf{D}_{ij}^{(1)}=\dots=\mathbf{D}_{ij}^{(K)}=0,(i,j)\notin J_{fixed}$, we only consider the variables in the set $J_{fixed}$. According to Lemma~\ref{thm:linearsy}, it is easy to see that $\mathbf{D}_{J_{fixed}}=0$ satisfies the optimality condition of the following problem
\[
   \min_{\mathbf{D}_{J_{fixed}}} \sum_{k=1}^K \vec(\widetilde {\mathbf{G}}_{t,J_{fixed}}^{(k)})^T\vec(\mathbf{D}_{J_{fixed}}^{(k)})+ P(\mathbf{D}_{J_{fixed}}).
\]
Since $\sum_{k=1}^K\vec(\mathbf{D}^{(k)})^T\mathbf{H}_t^{(k)}\vec(\mathbf{D}^{(k)}) \geq 0$, the optimal solution of \eqref{eq:fmgl_qp_app_const} is given by $\mathbf{D}^{(1)}=\dots=\mathbf{D}^{(K)}=0$.
\end{proof}

{Lemma~\ref{thm:fgml_lscreening} provides a shrinking scheme to partition the variables into the free set $J_{free}$ and the fixed set $J_{fixed}$. 
With shrinking, each Newton step of the proposed second-order method falls into a block coordinate gradient descent framework \cite{tseng2009coordinate}. Lemma \ref{thm:fgml_lscreening} shows that when the variables in the free set $J_{free}$ are fixed, no update is needed for the variables in the fixed set $J_{fixed}$. Minimization of \eqref{eq:fmgl_qp_app} restricted to the free set can therefore guarantee the convergence to the unique optimal solution \cite{hsieh2011sparse,tseng2009coordinate}.
In addition, it has been shown that local quadratic convergence rate can be achieved when the exact Hessian is used (see, for example, \cite{hsieh2011sparse,lee2012proximal}).}

{The resulting second-order method for solving the fused multiple graphical lasso is summarized in Algorithm \ref{alg:fusedMultiGraph}.}

\begin{algorithm}[h]\label{alg:fusedMultiGraph}
    \KwIn{$\mathbf{S}^{(k)},k=1,\dots, K,\lambda_1,\lambda_2$}
    \KwOut{$\mathbf{\Theta}^{(k)}, k=1,\dots,K$}
    Initialization: $\mathbf{\Theta}^{(k)}_0=(\mbox{diag}(\mathbf{S}^{(k)}))^{-1}$\;
  \While{Not Converged}{
  Determine the sets of free and fixed indices $J_{free}$ and $J_{fixed}$ using Lemma \ref{thm:fgml_lscreening}.\\
   Compute the Newton direction $\mathbf{D}^{(k)}, k=1,\dots, K$ by solving \eqref{eq:fmgl_qp_app} and \eqref{eq:fmgl_newton_d} over the free variables $J_{free}$.\\
   Choose $\mathbf{\Theta}^{(k)}_{t+1}$ by performing the Armijo backtracking line search along $\mathbf{\Theta}^{(k)}_t+\beta \mathbf{D}^{(k)}$ for $k=1,\dots,K$.
    }
    return {$\mathbf{\Theta}^{(k)}, k=1,\dots,K$}\;
    \caption{Proposed second-order method for Fused Multiple Graphical Lasso (FMGL)}
\end{algorithm}

\section{Experimental results\label{sec:expR}}
In this section, we evaluate the proposed algorithm and screening rule on synthetic datasets and two real datasets: ADHD-200\footnote{\url{http://fcon_1000.projects.nitrc.org/indi/adhd200/}} and FDG-PET images\footnote{\url{http://adni.loni.ucla.edu/}}. The experiments are performed on a PC with quad-core Intel 2.67GHz CPU and 9GB memory.

\subsection{Simulation}

\subsubsection{Efficiency}
We conduct experiments to demonstrate the effectiveness of the proposed screening rule and the efficiency of our method FMGL. The following algorithms are included in our comparisons:
\begin{itemize}
  \item FMGL: the proposed second-order method in Algorithm~\ref{alg:fusedMultiGraph}.
  \item ADMM: ADMM method.
  \item FMGL-S: FMGL with screening.
  \item ADMM-S: ADMM with screening.
\end{itemize}
Both FMGL and ADMM are written in Matlab. Since both methods involve solving~\eqref{eq:fmgl_newton_prox} which involves a double loop, we implement the sub-routine for solving~\eqref{eq:fmgl_newton_prox} in C for a fair comparison.

The synthetic covariance matrices are generated as follows. We first generate $K$ block diagonal ground truth precision matrices $\mathbf{\Theta}^{(k)}$ with $L$ blocks, and each block $\mathbf{\Theta}_l^{(k)}$ is of size $(p/L)\times (p/L)$. Each $\mathbf{\Theta}_l^{(k)},l=1,\dots,L, k=1,\dots,K$ has random sparsity structures. We control the number of nonzeros in each $\mathbf{\Theta}_l^{(k)}$ to be about $10p/L$ so that the total number of nonzeros in the $K$ precision matrices is $10Kp$. Given the precision matrices, we draw $5p$ samples from each Gaussian distribution to compute the sample covariance matrices. The fused penalty parameter $\lambda_2$ is fixed to 0.1, and the $\ell_1$ regularization parameter $\lambda_1$ is selected so that the total number of nonzeros in the solution is about $10Kp$. We terminate the NSPG in the FMGL when the relative error $\frac{\max\{\|\mathbf{\Theta}^{(k)}_r-\mathbf{\Theta}^{(k)}_{r-1}\|_\infty\}}{\max\{\|\mathbf{\Theta}^{(k)}_{r-1}\|_\infty\}}\leq 1e\mbox{-6}.$ The FMGL is terminated when the relative error of the objective value is smaller than $1e\mbox{-5}$, and ADMM stops until it achieves an objective value equal to or smaller than that of FMGL. The results presented in Table~\ref{tb:efficiency} show that FMGL is consistently faster than ADMM. FMGL converges much more quickly than ADMM. Moreover, the screening rule can achieve great computational gain. The speedup with the screening rule is about 10 and 20 times for $L=5$ and $10$ respectively.
\begin{table*}[h!]
\caption{Comparison of the proposed FMGL and ADMM with and without screening in terms of average computational time (seconds). FMGL-S and ADMM-S are FMGL and ADMM with screening respectively. $p$ stands for the dimension, $K$ is the number of graphs, $L$ is the number of blocks, and $\lambda_1$ is the $\ell_1$ regularization parameter. The fused penalty parameter $\lambda_2$ is fixed to 0.1. $\|\mathbf{\Theta}\|_0$ represents the total number of nonzero entries in ground truth precision matrices $\mathbf{\Theta}^{(k)}, k = 1,\dots,K$, and $\|\mathbf{\Theta}^*\|_0$ is the number of nonzeros in the solution.}
\centering\label{tb:efficiency}
\resizebox{\linewidth}{!}{%
\begin{tabular}{ |l|c|c|c|c|c|c|c|c|c| }
\hline
\multicolumn{6}{ |c| }{Data and parameter setting} & \multicolumn{4}{ |c| }{Computational time (iteration numbers)}\\
\hline
$p$ & $K$ & $L$ & $\|\mathbf{\Theta}\|_0$ & $\lambda_1$ & $\|\mathbf{\Theta}^*\|_0$ & FMGL-S & FMGL & ADMM-S & ADMM\\
\hline
500 & \multirow{2}{*}{{2}} & \multirow{6}{*}{{5}} & 9766 & 0.08 & 10228 &{\bf 0.86 }& 11.19 (6)& 13.78 & 98.89 (152)\\
\cline{1-1} \cline{4-10} 1000 & & &  19832 & 0.088 & 19322 & {\bf 6.21}& 57.78 (6)& 58.75& 529.36 (140)\\
\cline{1-2} \cline{4-10} 500 & \multirow{2}{*}{{5}} & & 24494 & 0.055 & 23878 & {\bf 2.62} & 33.15 (6) & 34.85  & 256.33 (146)\\
\cline{1-1} \cline{4-10}  1000 & & & 50836 & 0.054 & 44724 & {\bf 14.70} & 197.53 (6) & 171.68 & 1431.91 (150) \\
\cline{1-2} \cline{4-10} 500 & \multirow{2}{*}{{10}} & & 49500 & 0.051 & 45756 & {\bf 6.01} & 70.87 (6) & 73.42  & 524.84 (152) \\
\cline{1-1} \cline{4-10}  1000 & & & 100292 & 0.046 & 86774 & {\bf 30.49} & 383.46 (6) & 357.34 & 2991.16 (155) \\ \hline\hline
500 & \multirow{2}{*}{{2}} & \multirow{6}{*}{{10}} & 9528 & 0.07 & 9884 & {\bf 0.81} & 16.25 (7)& 5.01 & 109.24 (155)\\
\cline{1-1} \cline{4-10} 1000 & & &  19658 & 0.08 & 20612 & {\bf 1.65}& 75.44 (6)& 25.89& 560.75 (155)\\
\cline{1-2} \cline{4-10} 500 & \multirow{2}{*}{{5}} & & 23562 & 0.055 & 23600 & {\bf 1.69} & 47.32 (7) & 11.17 & 261.00 (153)\\
\cline{1-1} \cline{4-10}  1000 & & & 49274 & 0.054 & 46582 & {\bf 5.72} & 207.97 (6) & 75.61 & 1661.24 (172) \\
\cline{1-2} \cline{4-10} 500 & \multirow{2}{*}{{10}} & & 47364 & 0.051 & 48360 & {\bf 3.70} & 103.54 (6) & 24.75& 552.09 (157) \\
\cline{1-1} \cline{4-10}  1000 & & & 98650 & 0.046 & 96216 & {\bf 12.16} & 409.94 (6) & 150.62 & 3192.02 (168) \\ \hline
\end{tabular}}
\end{table*}

\subsubsection{Stability}
We conduct experiments to demonstrate the effectiveness of FMGL. The synthetic sparse precision matrices are generated in the following way: we set the first precision matrix $\mathbf{\Theta}^{(1)}$ as $0.25I_{p\times p}$, where $p=100$. When adding an edge $(i,j)$ in the graph, we add $\sigma$ to $\theta_{ii}^{(1)}$ and $\theta_{jj}^{(1)}$, and subtract $\sigma$ from $\theta_{ij}^{(1)}$ and $\theta_{ji}^{(1)}$ to keep the positive definiteness of $\mathbf{\Theta}^{(1)}$, where $\sigma$ is uniformly drawn from $[0.1,0.3]$. When deleting an edge $(i,j)$ from the graph, we reverse the above steps with $\sigma=\theta^{(1)}_{ij}$. We randomly assign 200 edges for $\mathbf{\Theta}^{(1)}$. $\mathbf{\Theta}^{(2)}$ is obtained by adding 25 edges and deleting 25 different edges from $\mathbf{\Theta}^{(1)}$. $\mathbf{\Theta}^{(3)}$ is obtained from $\mathbf{\Theta}^{(2)}$ in the same way. For each precision matrix, we randomly draw $n$ samples from the Gaussian distribution with the corresponding precision matrix, where $n$ varies from 40 to 200 with a step of 20. We perform 500 replications for each $n$. For each $n$, $\lambda_2$ is fixed to 0.08, and $\lambda_1$ is adjusted to make sure that the edge number is about 200. The accuracy $n_d/n_g$ is used to measure the performance of FMGL and GLasso, where $n_d$ is the number of true edges detected by FGML and GLasso, and $n_g$ is the number of true edges. The results are shown in Figure~\ref{fig:syProb}. We can see from the figure that FMGL achieves higher accuracies, demonstrating the effectiveness of FMGL for learning multiple graphical models simultaneously.
\begin{figure}[ht]
\centering
\includegraphics[width=0.325\textwidth]{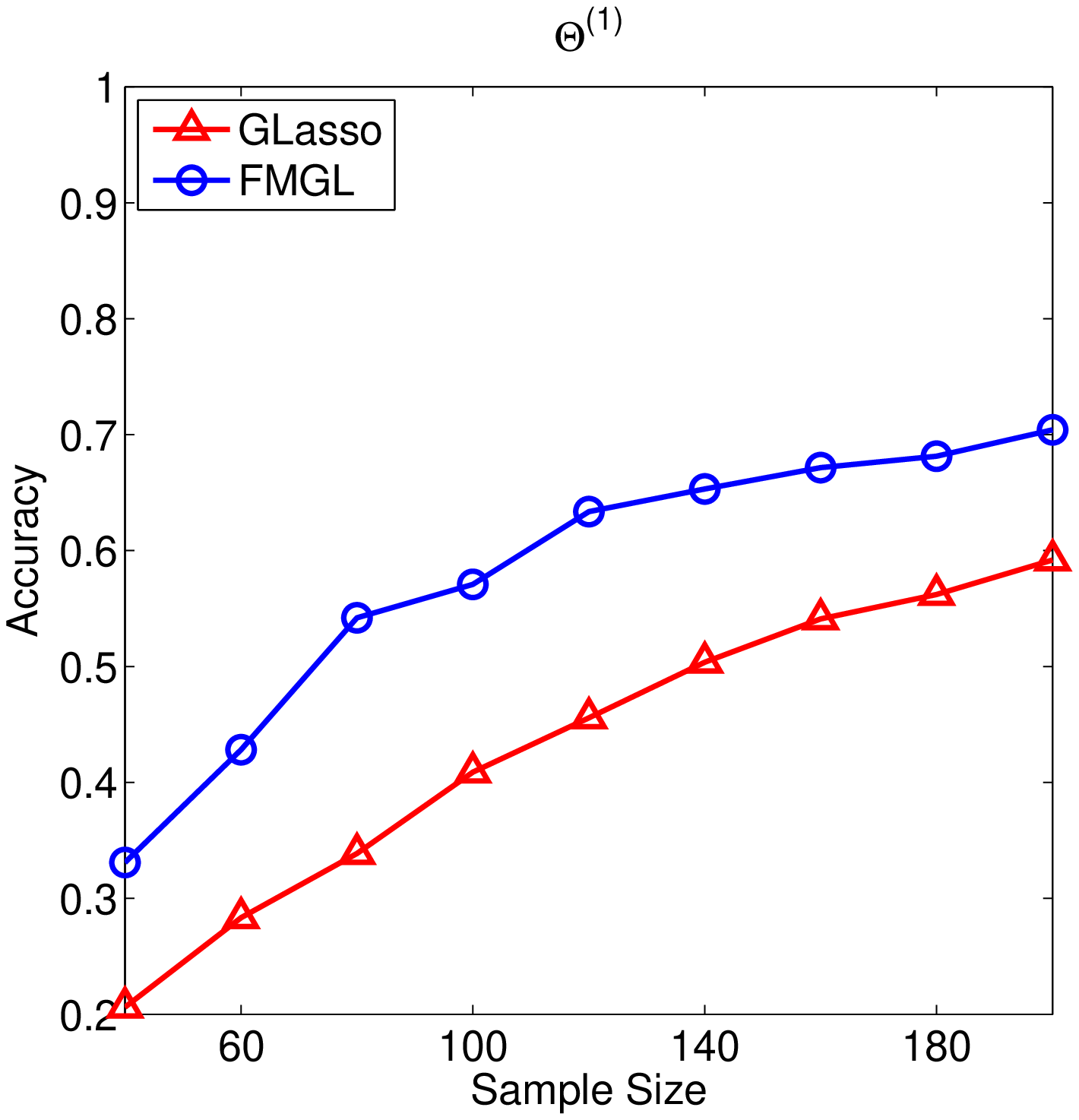}
\includegraphics[width=0.325\textwidth]{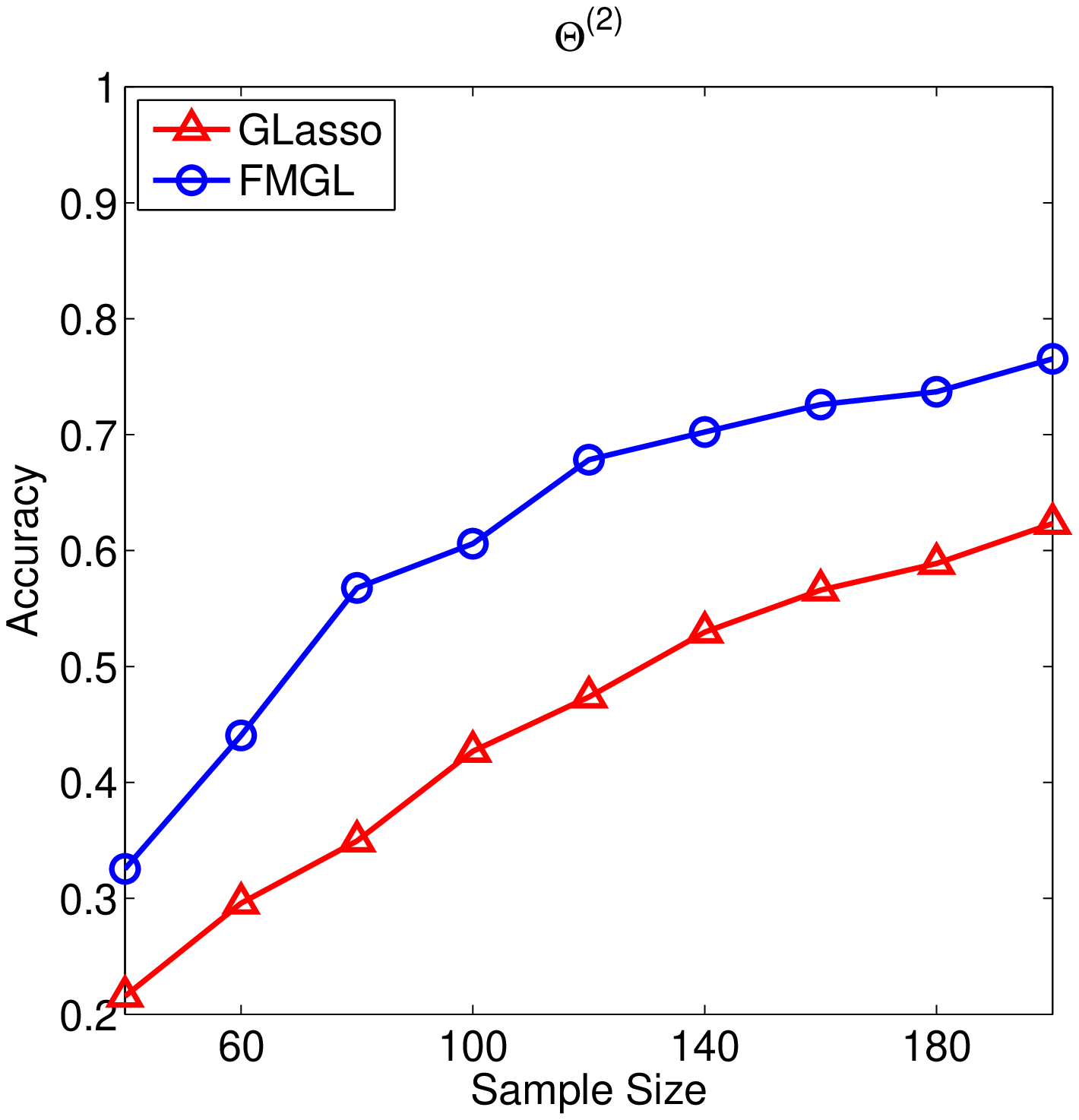}
\includegraphics[width=0.325\textwidth]{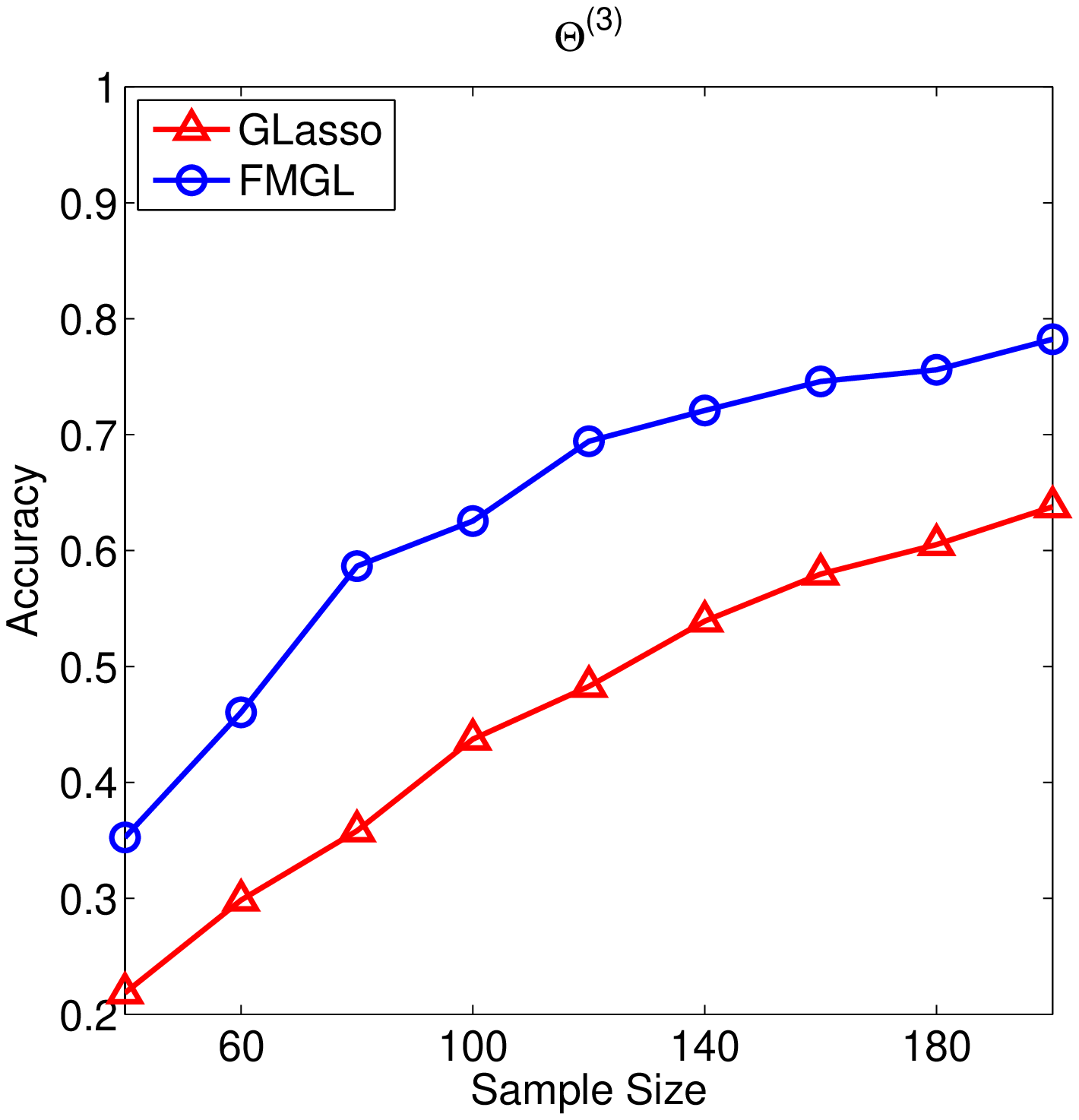}
\caption{Comparison of FMGL and GLasso in detecting true edges. {Sample size varies from 40 to 200 with a step of 20.}}\label{fig:syProb}
\vspace{-0.1in}
\end{figure}

\subsection{Real data}
\subsubsection{ADHD-200}
Attention Deficit Hyperactivity Disorder (ADHD) affects at least 5-10\% of school-age children with annual costs exceeding 36 billion/year in the United States. The ADHD-200 project has released resting-state functional magnetic resonance images (fMRI) of 491 typically developing children and 285 ADHD children, aiming to encourage the research on ADHD. The data used in this experiment is preprocessed using the NIAK pipeline, and downloaded from neurobureau\footnote{\url{http://www.nitrc.org/plugins/mwiki/index.php?title=neurobureau:NIAKPipeline/}}. More details about the preprocessing strategy can be found in the same website. The dataset we choose includes 116 typically developing children (TDC), 29 ADHD-Combined (ADHD-C), and 49 ADHD-Inattentive (ADHD-I). There are 231 time series and 2834 brain regions for each subject. We want to estimate the graphs of the three groups simultaneously. The sample covariance matrix is computed using all data from the same group. Since the number of brain regions $p$ is 2834, obtaining the precision matrices is computationally intensive. We use this data to test the effectiveness of the proposed screening rule.
$\lambda_1$ and $\lambda_2$ are set to 0.6 and 0.015. The convergence criterion is 1e-5. The comparison of FMGL and ADMM in terms of the objective value curve is shown in Figure~\ref{fig:adhdobjc}. The result shows that FMGL converges much faster than ADMM. The computational times of FMGL and ADMM are 1557.08 and 8306.35 seconds respectively. However, utilizing the screening, the computational times of FMGL-S and ADMM-S are 18.08 and 119.23 seconds respectively, demonstrating the superiority of the screening rule. The obtained solution has 1443 blocks. The largest one including 634 nodes is shown in Figure~\ref{fig:subgraph}.

\begin{figure}[h!]
\centering
\includegraphics[width=0.45\textwidth]{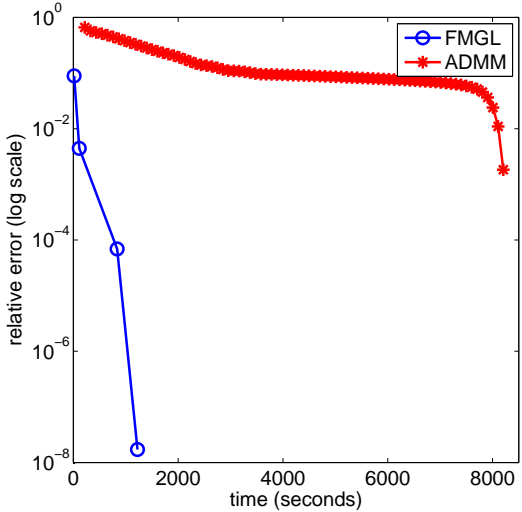}
\caption{{Comparison of FMGL and ADMM in terms of objective value curve on the ADHD-200 dataset. The dimension $p$ is 2834, and the number of graphs $K$ is 3.}}\label{fig:adhdobjc}
\end{figure}

\begin{figure}[h!]
\centering
\includegraphics[width=0.90\textwidth]{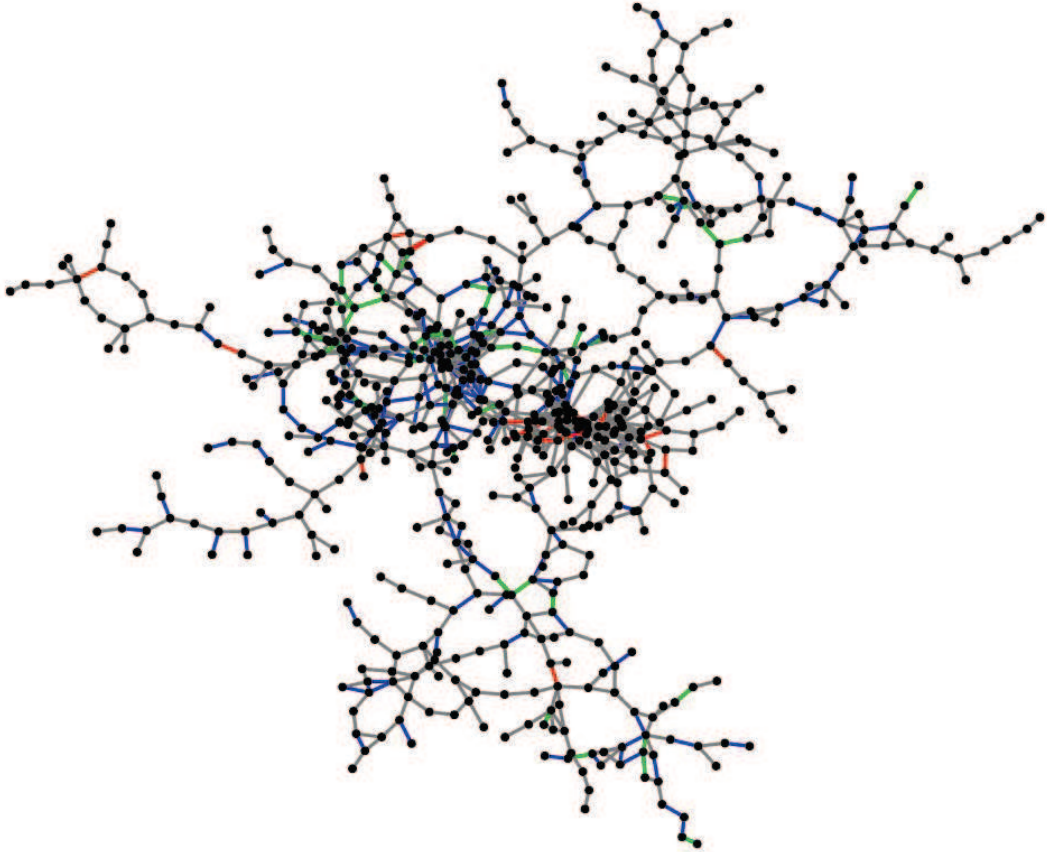}
\caption{A subgraph of ADHD-200 identified by FMGL with the proposed screening rule. The grey edges are common edges among the three graphs; the red, green, and blue edges are the specific edges for TDC, ADHD-I, and ADHD-C respectively.}\label{fig:subgraph}
\end{figure}

\begin{figure}[ht!]
\centering
\subfigure[]{
\includegraphics[width=0.45\textwidth]{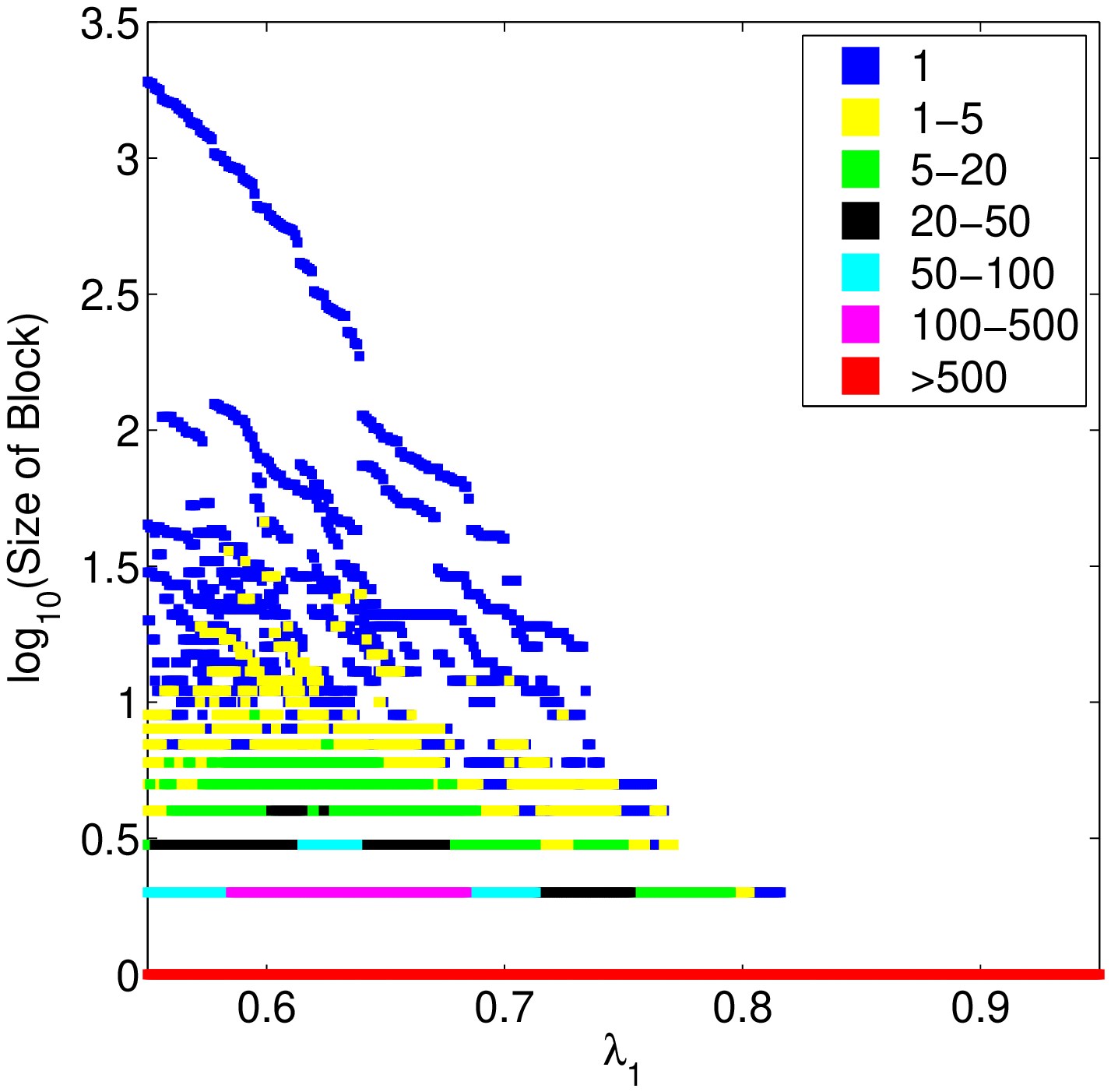}\label{fig:connect:f1} 
}
\subfigure[]{
\includegraphics[width=0.457\textwidth]{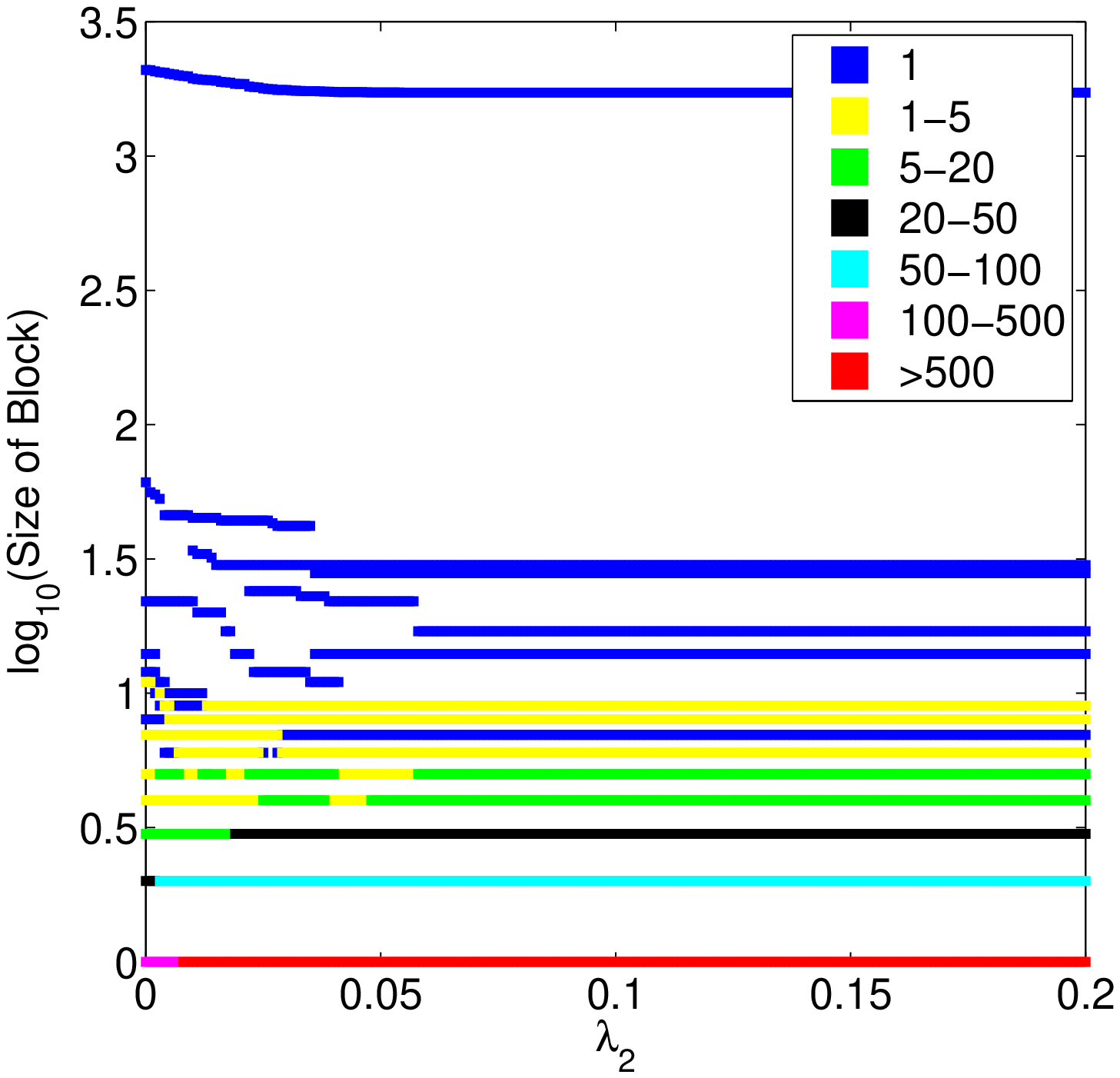}
\label{fig:connect:f2}
}
\caption{The size distribution of blocks (in the logarithmic scale) identified by the proposed screening rule. The color represents the number of blocks of a specified size. (a): $\lambda_1$ varies from 0.5 to 0.95 with $\lambda_2$ fixed to 0.015. (b): $\lambda_2$ varies from 0 to 0.2 with $\lambda_1$ fixed to 0.55.}\label{fig:connect}
\end{figure}

The block structures of the FMGL solution are the same as those identified by the screening rule. The screening rule can be used to analyze the rough structures of the graphs. The cost of identifying blocks using the screening rule is negligible compared to that of estimating the graphs. For high-dimensional data such as ADHD-200, it is practical to use the screening rule to identify the block structure before estimating the large graphs. We use the screening rule to identify block structures on ADHD-200 data with varying $\lambda_1$ and $\lambda_2$. The size distribution is shown in Figure~\ref{fig:connect}. We can observe that the number of blocks increases, and the size of blocks deceases when the regularization parameter value increases.

\subsubsection{FDG-PET}
In this experiment, we use FDG-PET images from 74 Alzhei-mer's disease (AD), 172 mild cognitive impairment (MCI), and 81 normal control (NC) subjects downloaded from the Alzheimer's disease neuroimaging initiative (ADNI) database. The different regions of the whole brain volume can be represented by 116 anatomical volumes of interest (AVOI), defined by Automated Anatomical Labeling (AAL)~\cite{tzourio2002automated}. Then we extracted data from each of the 116 AVOIs, and derived the average of each AVOI for each subject. The 116 AVOIs can be categorized into 10 groups: prefrontal lobe, other parts of the frontal lobe, parietal lobe, occipital lobe, thalamus, insula, temporal lobe, corpus striatum, cerebellum, and vermis. More details about the categories can be found in \cite{tzourio2002automated,wang2007altered}. We remove two small groups (thalamus and insula) containing only 4 AVOIs in our experiments.

To examine whether FMGL can effectively utilize the information of common structures, we randomly select $g$ percent samples from each group, where $g$ varies from 20 to 100 with a step size of 10. For each $g$, $\lambda_2$ is fixed to 0.1, and $\lambda_1$ is adjusted to make sure the number of edges in each group is about the same. We perform 500 replications for each $g$. The edges with probability larger than 0.85 are considered as stable edges. The results showing the numbers of stable edges are summarized in Figure~\ref{fig:probt}. We can observe that FMGL is more stable than GLasso. When the sample size is too small (say 20\%), there are only 20 stable edges in the graph of NC obtained by GLasso. But the graph of NC obtained by FMGL still has about 140 stable edges, illustrating the superiority of FMGL in stability.
\begin{figure}[h!]
\centering
\hspace{-1mm}\includegraphics[width=0.33\textwidth]{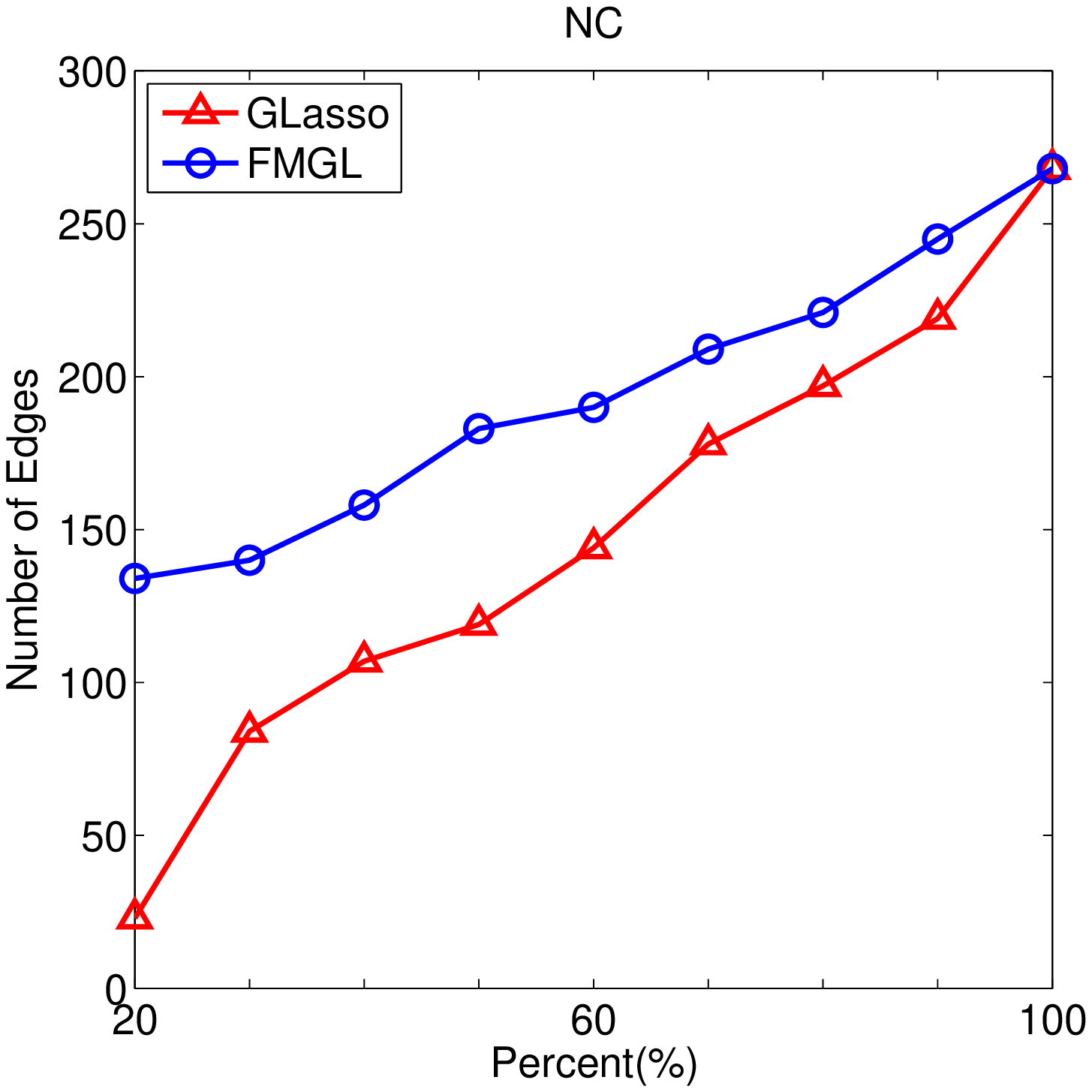}
\hspace{-2mm}
\includegraphics[width=0.33\textwidth]{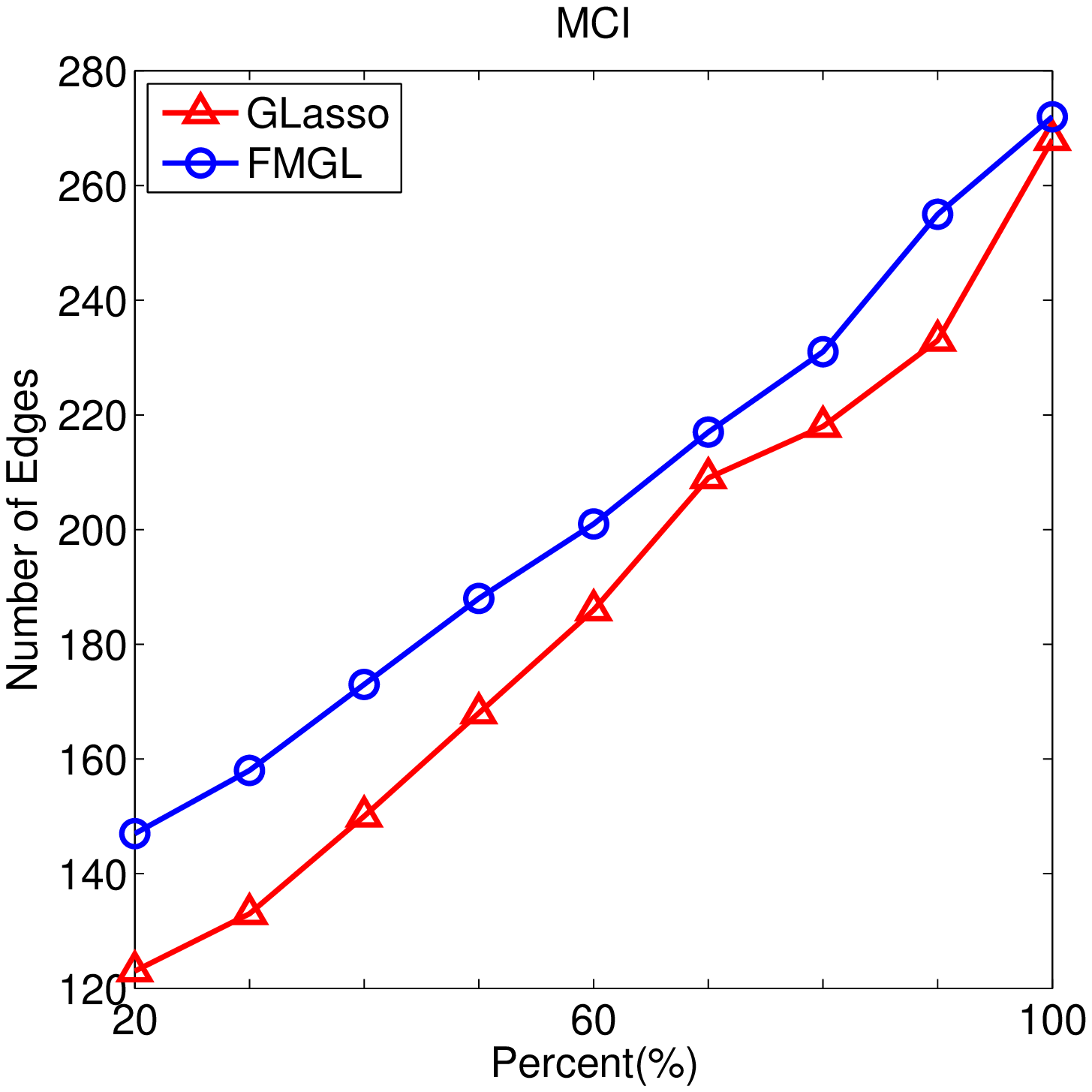}
\hspace{-2mm}
\includegraphics[width=0.33\textwidth]{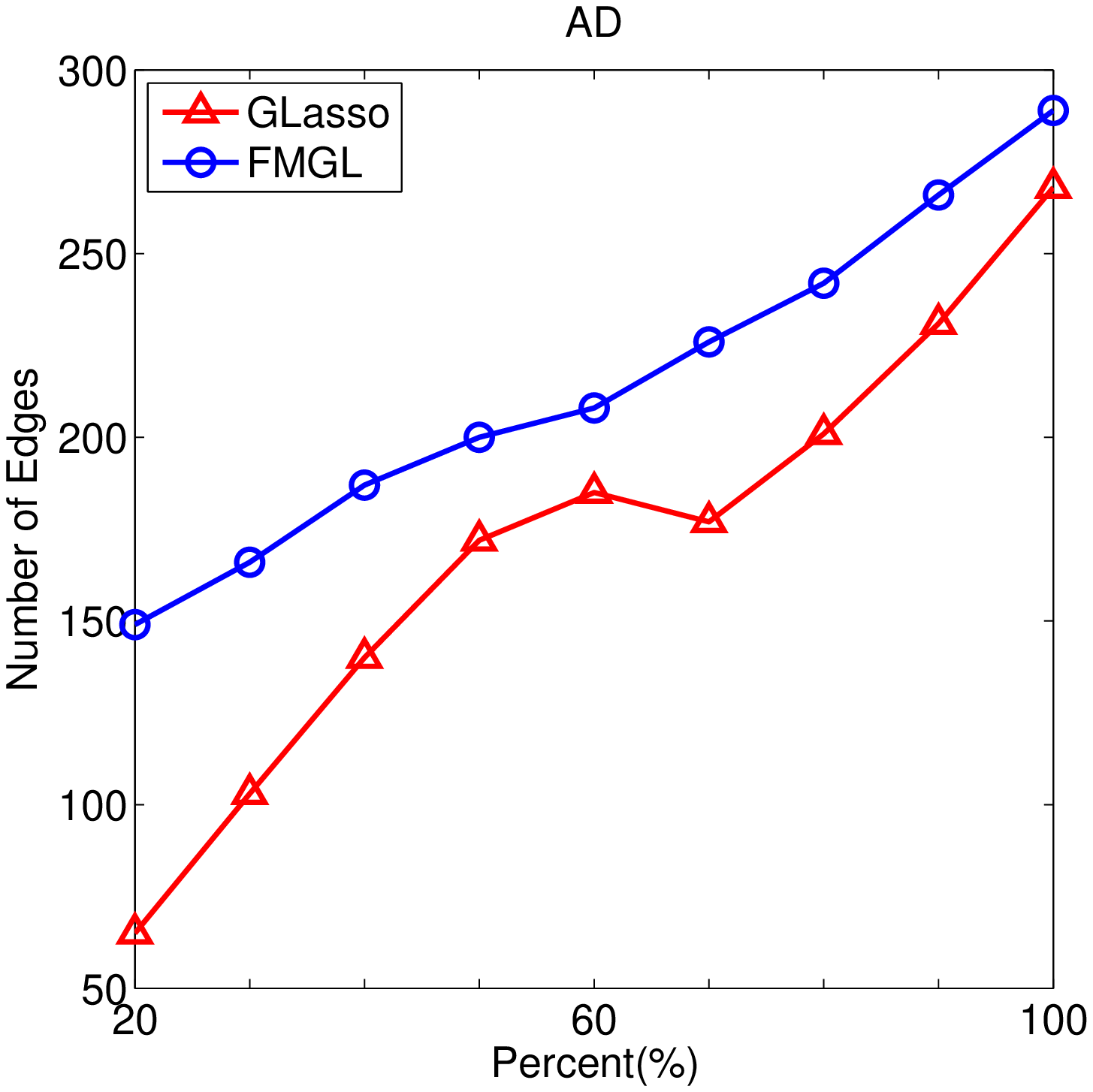}
\caption{The average number of stable edges detected by FMGL and GLasso in NC, MCI, and AD of 500 replications. Sample size varies from 20\% to 100\% with a step of 10\%.}\label{fig:probt}
\end{figure}
\begin{figure}[ht!]
\centering
\includegraphics[width=0.32\textwidth]{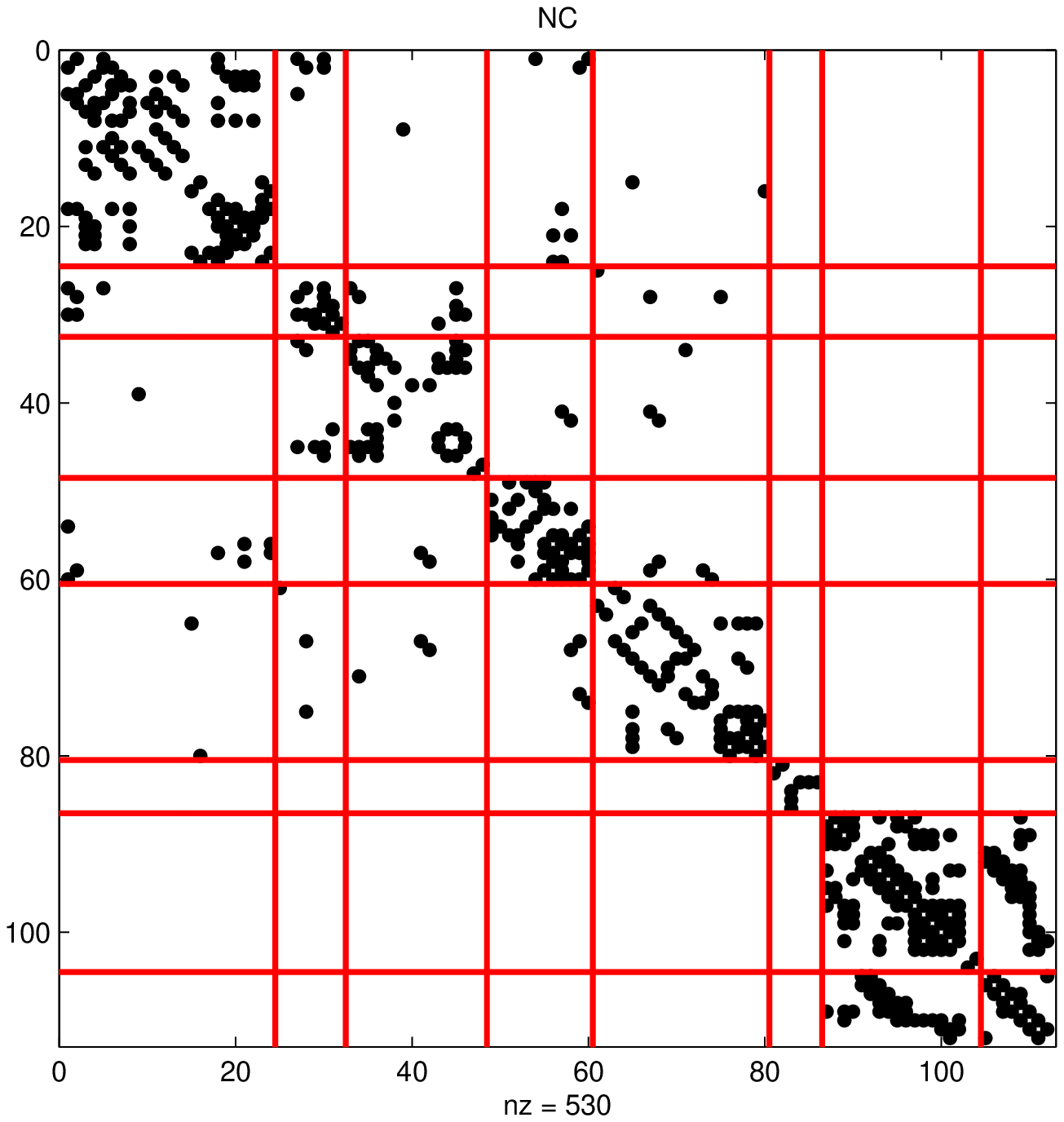}
\includegraphics[width=0.32\textwidth]{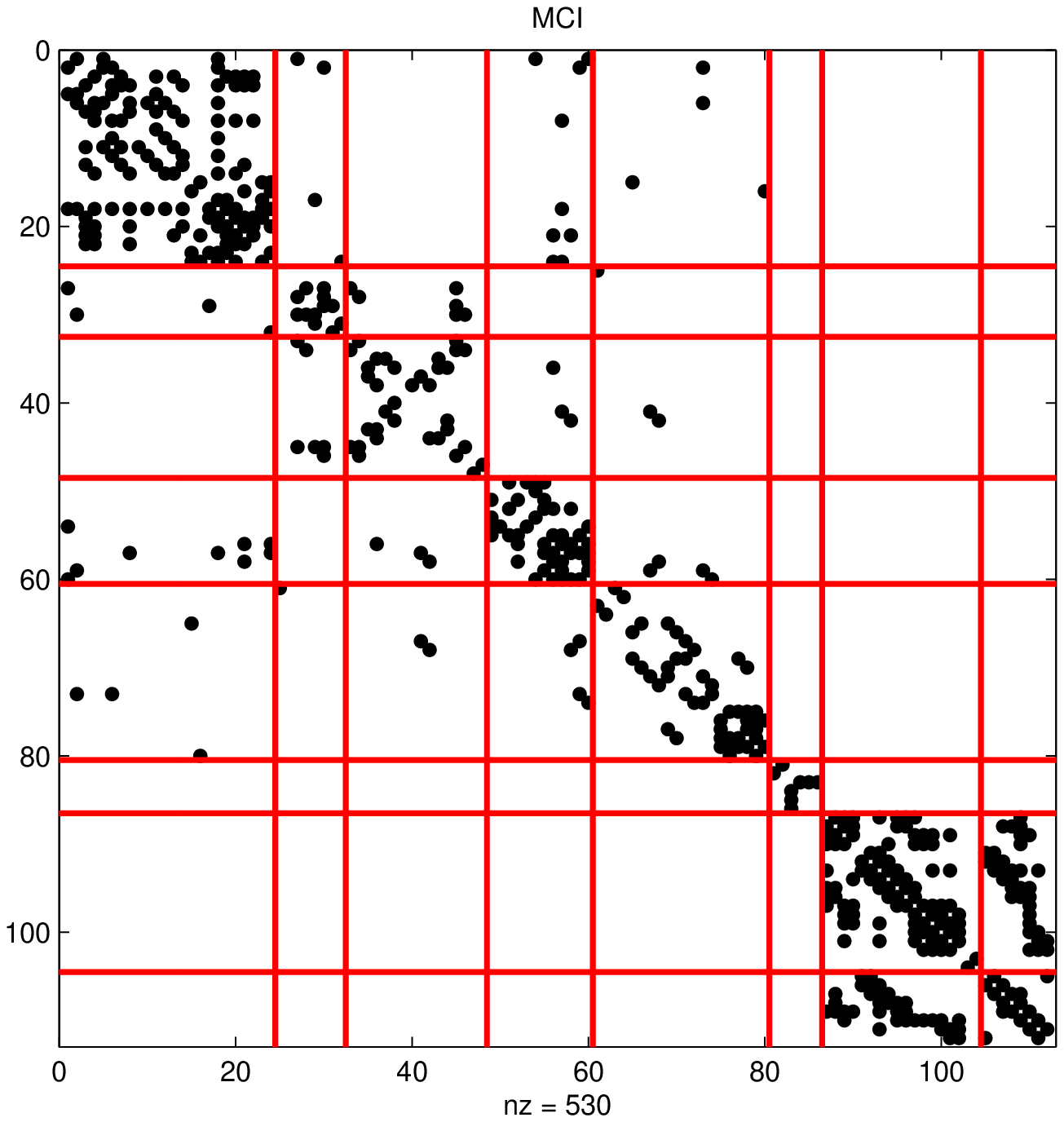}
\includegraphics[width=0.32\textwidth]{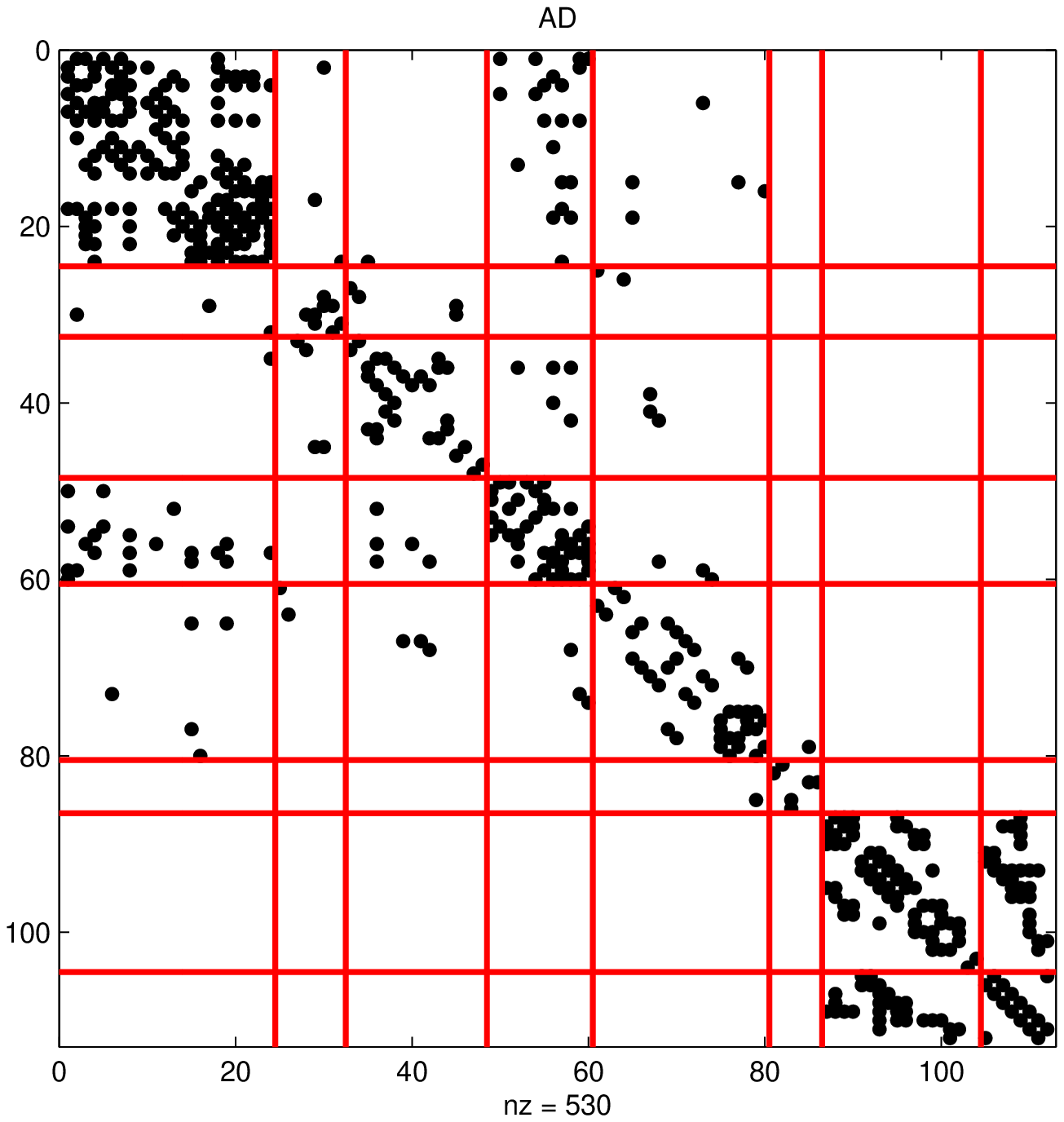}
\caption{Brain connection models with 265 edges: NC, MCI, and AD. In each figure, the diagonal blocks are prefrontal lobe, other parts of frontal lobe, parietal lobe, occipital lobe, temporal lobe, corpus striatum, cerebellum, and vermis respectively.}\label{fig:ADNI}
\end{figure}

The brain connectivity models obtained by FMGL are shown in Figure~\ref{fig:ADNI}. We can see that the number of connections within the prefrontal lobe significantly increases, and the number of connections within the temporal lobe significantly decreases from NC to AD, which are supported by previous literatures~\cite{azari1992patterns,horwitz1987intercorrelations}. The connections between the prefrontal and occipital lobes increase from NC to AD, and connections within cerebellum decrease. We can also find that the adjacent graphs are similar, indicating that FMGL can identify the common structures, but also keep the meaningful differences.

\section{Conclusion}\label{sec:discussion}
In this paper, we consider simultaneously estimating multiple graphical models by maximizing a fused penalized log likelihood. We have derived a set of necessary and sufficient conditions for the FMGL solution to be block diagonal for an arbitrary number of graphs. A screening rule has been developed to enable the efficient estimation of large multiple graphs. The second-order method is employed to solve the fused multiple graphical lasso. The global convergence of the proposed method is guaranteed, and the convergence rate is local quadratic. A shrinking scheme is proposed to identify the variables to be updated during the Newton iterations, thus reduces the computation. Numerical experiments on synthetic and real data demonstrate the efficiency and effectiveness of the proposed method and the screening rule. We plan to explore the convergence properties of the second-order method using the inexact Newton direction. Due to the shrinking scheme, the proposed second-order method is suitable for warm-start techniques. A good initial solution can further speedup the computation. As part of the future work, we plan to explore how to efficiently find a good initial solution to further improve the efficiency of the proposed method. One possibility is to use divide-and-conquer techniques~\cite{hsieh2012divide}.
\bibliographystyle{siam}	
\small{\bibliography{nips2012}}
\end{document}